\documentclass{article}

\usepackage{microtype}
\usepackage{graphicx}
\usepackage{booktabs} 

\usepackage{caption}
\usepackage{subcaption}

\usepackage{hyperref}

\usepackage{hyperref}
\usepackage{url}

\usepackage[utf8]{inputenc} 
\usepackage[T1]{fontenc}    
\usepackage{booktabs}       
\usepackage{amsfonts}       
\usepackage{nicefrac}       

\usepackage{amsfonts}       
\usepackage{amsmath}
\usepackage{amssymb}
\usepackage{amsthm}

\usepackage{mathabx}

\usepackage{wrapfig}

\usepackage{multirow}

\usepackage{enumerate}

\usepackage{thm-restate}
\usepackage[toc,page]{appendix}

\usepackage{algorithm}

\usepackage{chngcntr}
\usepackage{apptools}
\AtAppendix{\counterwithin{lemma}{section}}
\AtAppendix{\counterwithin{proposition}{section}}
\AtAppendix{\counterwithin{corollary}{section}}

\DeclareMathOperator*{\E}{\mathbb{E}}

\DeclareMathOperator*{\argmax}{argmax}

\DeclareMathOperator{\Tr}{trace}

\newcommand{\wb}{\widebar}

\newcommand{\cL}{\mathcal{L}}

\newcommand{\cS}{\mathcal{S}}

\newcommand{\cA}{\mathcal{A}}

\newcommand{\1}{\boldsymbol{1}}

\newcommand{\R}{\mathbb{R}}

\newcommand{\pik}{\pi_{k}}
\newcommand{\pikup}{\pi_{k+1}}
\newcommand{\pol}{\pi_{\theta}}

\newcommand{\polk}{\pi_{\theta_k}}
\newcommand{\polkup}{\pi_{\theta_{k+1}}}
\newcommand{\grad}{\nabla_{\theta}}

\newcommand{\vfunc}{\wb{V}^{\pi}}
\newcommand{\qfunc}{\wb{Q}^{\pi}}

\newcommand{\vphi}{\wb{V}^{\pi}_{\phi}}
\newcommand{\adv}{\wb{A}^{\pi}}
\newcommand{\vfuncd}{V_{\gamma}^{\pi}}
\newcommand{\qfuncd}{Q_{\gamma}^{\pi}}

\newcommand{\advd}{A_{\gamma}^{\pi}}
\newcommand{\vtarg}{\wb{V}^{\text{target}}}

\newcommand{\advk}{\wb{A}^{\pik}}

\newcommand{\dpi}{d_{\pi}}
\newcommand{\dpip}{d_{\pi'}}
\newcommand{\dpid}{d_{\pi,\gamma}}

\newcommand{\dpik}{d_{\pik}}
\newcommand{\dpolk}{d_{\pi_{\theta_k}}}
\newcommand{\dpolkd}{d_{\pi_{\theta_k,\gamma}}}

\newcommand{\Mpi}{M^{\pi}}

\newcommand{\Zpi}{Z^{\pi}}
\newcommand{\Mpip}{M^{\pi'}}

\newcommand{\Zpip}{Z^{\pi'}}
\newcommand{\dg}{\text{dg}}

\newcommand{\norm}[1]{\left\| #1 \right\|}
\newcommand{\KL}[2]{D_\text{KL}\left(#1\middle\| #2\right)}

\newcommand{\avKL}[2]{\bar{D}_\text{KL}(#1\parallel #2)}

\newcommand{\TV}[2]{D_\text{TV}(#1\parallel #2)}

\newtheorem{theorem}{Theorem}
\newtheorem{lemma}{Lemma}

\newtheorem{proposition}{Proposition}
\newtheorem{corollary}{Corollary}

\newtheorem{assumption}{Assumption}

\usepackage[accepted]{icml2021}

\icmltitlerunning{On-Policy Deep Reinforcement Learning for the Average-Reward Criterion}

\begin{document}

\twocolumn[
\icmltitle{On-Policy Deep Reinforcement Learning for the Average-Reward Criterion}




\begin{icmlauthorlist}
\icmlauthor{Yiming Zhang}{to}
\icmlauthor{Keith W. Ross}{goo,to}
\end{icmlauthorlist}

\icmlaffiliation{to}{New York University}
\icmlaffiliation{goo}{New York University Shanghai}

\icmlcorrespondingauthor{Yiming Zhang}{yiming.zhang@cs.nyu.edu}

\icmlkeywords{Machine Learning, ICML}

\vskip 0.3in
]



\printAffiliationsAndNotice{}  

\begin{abstract}
We develop theory and algorithms for average-reward on-policy Reinforcement Learning (RL). We first consider bounding the difference of the long-term average reward for two policies. We show that previous work based on the discounted return \citep{schulman2015trust,achiam2017constrained} results in a non-meaningful bound in the average-reward setting. By addressing the average-reward criterion directly, we then derive a novel bound which depends on the average divergence between the two policies and Kemeny's constant. Based on this bound, we develop an iterative procedure which produces a sequence of monotonically improved policies for the average reward criterion. This iterative procedure can then be combined with classic DRL (Deep Reinforcement Learning) methods, resulting in practical DRL algorithms that target the long-run average reward criterion. In particular, we demonstrate that Average-Reward TRPO (ATRPO), which adapts the on-policy TRPO algorithm to the average-reward criterion, significantly outperforms TRPO in the most challenging MuJuCo environments.
\end{abstract}

\section{Introduction}
The goal of Reinforcement Learning (RL) is to build agents that can learn high-performing behaviors through trial-and-error interactions with the environment. Broadly speaking, modern RL tackles two kinds of problems: \textit{episodic tasks} and \textit{continuing tasks}. In episodic tasks, the agent-environment interaction can be broken into separate distinct episodes, and the performance of the agent is simply the sum of the rewards accrued within an episode. Examples of episodic tasks include training an agent to learn to play Go \citep{silver2016mastering,silver2018general}, where the episode terminates when the game ends.
In continuing tasks, such as robotic locomotion \citep{peters2008reinforcement,schulman2015trust,haarnoja2018soft} or in a queuing scenario \cite{tadepalli1994h,sutton2018reinforcement}, there is no natural separation of episodes and the agent-environment interaction  continues indefinitely. The performance of an agent in a continuing task is more difficult to quantify since the total sum of rewards is typically infinite.

One way of making the long-term reward objective meaningful for continuing tasks is to apply \textit{discounting} so that the infinite-horizon return is guaranteed to be finite for any bounded reward function. However the discounted objective biases the optimal policy to choose actions that lead to high near-term performance rather than to high long-term performance.  Such an objective is not appropriate when the goal is to optimize long-term behavior, i.e., when the natural objective underlying the task at hand is non-discounted. In particular, we note that for the vast majority of benchmarks for reinforcement learning such as Atari games \citep{mnih2013playing} and MuJoCo \citep{todorov2012mujoco}, a non-discounted performance measure is used to evaluate the trained policies.

Although in many circumstances, non-discounted criteria are more natural, most 
of the successful DRL algorithms today have been designed to optimize a discounted criterion during training. One possible work-around for this mismatch is to simply train with a discount factor that is very close to one. Indeed, from the Blackwell optimality theory of MDPs \cite{blackwell1962discrete}, we know that if the discount factor is very close to one, then an optimal policy for the infinite-horizon discounted criterion is also optimal for the long-run average-reward criterion. However, although Blackwell's result suggests we can simply use a large discount factor to optimize non-discounted criteria, problems with large discount factors are in general more difficult to solve \citep{petrik2008biasing,jiang2015dependence,jiang2016structural,lehnert2018value}. Researchers have also observed that state-of-the-art DRL algorithms typically break down when the discount factor gets too close to one \citep{schulman2016high,andrychowicz2020matters}.  

In this paper we seek to develop algorithms for 
finding high-performing policies for average-reward DRL problems. Instead of trying to simply use standard discounted DRL algorithms with large discount factors, we instead attack the problem head-on, seeking to directly optimize the average-reward criterion. 
While the average reward setting has been extensively studied in the classical Markov Decision Process literature \citep{howard1960dynamic,blackwell1962discrete,veinott1966finding,bertsekas1995dynamic}, and has to some extent been studied for tabular RL \citep{schwartz1993reinforcement,mahadevan1996average,abounadi2001learning,wan2020learning}, it has received relatively little attention in the DRL community. In this paper, our focus is on developing average-reward on-policy DRL algorithms.

One major source of difficulty with modern on-policy DRL algorithms lies in controlling the step-size for policy updates. In order to have better control over step-sizes, \citet{schulman2015trust} constructed a lower bound on the difference between the expected discounted return
for two arbitrary policies $\pi$ and $\pi'$ by building upon the work of \citet{kakade2002approximately}. The bound is a function of the divergence between these two policies and the discount factor. \citet{schulman2015trust} showed that iteratively maximizing this lower bound generates a sequence of monotonically improved policies for their discounted return. 

In this paper, we first show that the policy improvement theorem from \citet{schulman2015trust} results in a non-meaningful bound in the average reward case. We then derive a novel result which lower bounds the difference of the average long-run rewards.
The bound depends on the average divergence between the policies and on the so-called Kemeny constant, which measures to what degree the irreducible Markov chains associated with the policies are ``well-mixed''. We show that iteratively maximizing this lower bound guarantees monotonic average reward policy improvement. 

Similar to the discounted case, the problem of maximizing the lower bound can be approximated with DRL algorithms which can be optimized using samples collected in the environment. In particular, we describe in detail the Average Reward TRPO (ATRPO) algorithm, which is the average reward variant of the TRPO algorithm \citep{schulman2015trust}. Using the MuJoCo simulated robotic benchmark, we carry out extensive experiments demonstrating the effectiveness of of ATRPO compared to its discounted counterpart, in particular on the most challenging MuJoCo tasks.
Notably, we show that ATRPO can significantly out-perform TRPO on a set of high-dimensional continuing control tasks.

Our main contributions can be summarized as follows:
\begin{itemize}
    \item We extend the policy improvement bound from \citet{schulman2015trust} and \citet{achiam2017constrained} to the average reward setting. We demonstrate that our new bound depends on the average divergence between the two policies and on the mixing time of the underlying Markov chain.
    \item We use the aforementioned policy improvement bound to derive novel on-policy deep reinforcement learning algorithms for optimizing the average reward.
    \item Most modern DRL algorithms introduce a discount factor during training even when the natural objective of interest is undiscounted. This leads to a discrepancy between the evaluation and training objective. We demonstrate that optimizing the average reward directly can effectively address this mismatch and lead to much stronger performance.
\end{itemize}

\section{Preliminaries}\label{sec:prelim}
Consider a Markov Decision Process (MDP) \citep{sutton2018reinforcement} $(\cS, \cA, P, r, \mu)$ where the state space $\cS$ and action space $\cA$ are assumed to be finite. The transition probability is denoted by $P:\cS\times\cA\times \cS\to[0,1]$, the bounded reward function  $r:\cS\times\cA\to [r_{\min}, r_{\max}]$, and $\mu:\cS\to [0,1]$ is the initial state distribution.  Let $\pi:\cS\to\Delta(\cA)$ be a stationary policy where $\Delta(\cA)$ is the probabilty simplex over $\cA$, and $\Pi$ is the set of all stationary policies. 
We consider two classes of MDPs:
\begin{assumption}[Ergodic]\label{assump:irreducible}
For every stationary policy, the induced Markov chain is irreducible and aperiodic.
\end{assumption}
\begin{assumption}[Aperiodic Unichain]\label{assump:unichain}
For every stationary policy, the induced Markov chain contains a single aperiodic recurrent class and a finite but possibly empty set of transient states.
\end{assumption}
By definition, any MDP which satisfies Assumption \ref{assump:irreducible} is also unichain.
We note that most MDPs of practical interest belong in these two classes. We will mostly focus on MDPs which satisfy Assumption \ref{assump:irreducible} in the main text. In the supplementary material, we will address the aperiodic unichain case. Here we present the two objective formulations for continuing control tasks: the average reward approach and discounted reward criterion.

\textbf{Average Reward Criterion} \\
The average reward objective is defined as:
\begin{equation}\label{eq:AvgR_obj}
    \rho(\pi) := \lim_{N\to\infty}\frac{1}{N}\E_{\tau\sim\pi}\left[\sum_{t=0}^{N-1} r(s_t,a_t)\right] = \E_{\substack{s\sim\dpi\\ a\sim\pi}}[r(s,a)].
\end{equation}
Here $\dpi(s):=\lim_{N\to\infty}\frac{1}{N} \sum_{t=0}^{N-1} P_{\tau\sim\pi}(s_t=s)$ is the \textit{stationary state distribution under policy $\pi$}, and $\tau=(s_0,a_0,\dots,)$ is a sample trajectory. The limits in $\rho(\pi)$ and $\dpi(s)$ are guaranteed to exist under our assumptions. Since the MDP is aperiodic, it can also be shown that $\dpi(s)=\lim_{t\to\infty}P_{\tau\sim\pi}(s_t=s)$. In the unichain case, the average reward $\rho(\pi)$ does not depend on the initial state for any policy $\pi$ \citep{bertsekas1995dynamic}.
We express the \textit{average-reward bias function} as
\begin{equation*}
\vfunc(s):= \E_{\tau\sim\pi}\left[\sum_{t=0}^{\infty} (r(s_t,a_t) - \rho(\pi))\bigg| s_0=s\right]
\end{equation*}
and \textit{average-reward action-bias function} as
\begin{equation*}
\qfunc(s,a):= \E_{\tau\sim\pi}\left[\sum_{t=0}^{\infty} (r(s_t,a_t) - \rho(\pi))\bigg| s_0=s, a_0=a\right].
\end{equation*}
We define the \textit{average-reward advantage function} as
\begin{equation*}
\adv(s,a) := \qfunc(s,a)-\vfunc(s).
\end{equation*}

\textbf{Discounted Reward Criterion} \\
For some discount factor $\gamma\in (0,1)$, the discounted reward objective is defined as
\begin{equation}
    \rho_{\gamma}(\pi) :=  \E_{\tau\sim\pi}\left[\sum_{t=0}^{\infty}\gamma^t r(s_t,a_t)\right] = \frac{1}{1-\gamma}\E_{\substack{s\sim\dpid\\ a\sim\pi}}[r(s,a)]
\end{equation}
where $\dpid(s) := (1-\gamma)\sum_{t=0}^{\infty}\gamma^t P_{\tau\sim\pi}(s_t=s)$ is known as the \textit{future discounted state visitation distribution under policy $\pi$}. Note that unlike the average reward objective, the discounted objective depends on the initial state distribution $\mu$. It can be easily shown that $\dpid(s)\to\dpi(s)$ for all $s$ as $\gamma\to 1$. The \textit{discounted value function} is defined as $\vfuncd(s):= \E_{\tau\sim\pi}\left[\sum_{t=0}^{\infty} \gamma^t r(s_t,a_t)\bigg| s_0=s\right]$ and \textit{discounted action-value function} $\qfuncd(s,a):= \E_{\tau\sim\pi}\left[\sum_{t=0}^{\infty} \gamma^t r(s_t,a_t)\bigg| s_0=s, a_0=a\right]$.
Finally, the \textit{discounted advantage function} is defined as $\advd(s,a) := \qfuncd(s,a)-\vfuncd(s)$.

It is well-known that $\lim_{\gamma\to 1}(1-\gamma)\rho_{\gamma}(\pi)=\rho(\pi)$, implying that the discounted and average reward objectives are equivalent in the limit as $\gamma$ approaches 1 \citep{blackwell1962discrete}. We further discuss the relationship between the discounted and average reward criteria in Appendix \ref{append:disc_avg} and prove that $\lim_{\gamma\to 1}\advd(s,a) = \adv(s,a)$ (see Corollary \ref{corollary:discount_relations}).
The proofs of all results in the subsequent sections, if not given, can be found in the supplementary material.

\section{Montonically Improvement Guarantees for Discounted RL}
In much of the on-policy DRL literature \citep{schulman2015trust,schulman2017proximal, wu2017scalable, vuong2019supervised, song2020v}, algorithms iteratively update policies by maximizing them within a local region, i.e., at iteration $k$ we find a policy $\pi_{k+1}$ by maximizing $\rho_{\gamma}(\pi)$ within some region $D(\pi,\pik)\leq\delta$ for some divergence measure $D$. By using different choices of $D$ and $\delta$, this approach allows us to control the step-size of each update,  which can lead to better sample efficiency \citep{peters2008reinforcement}.
\citet{schulman2015trust} derived a policy improvement bound based on a specific choice of $D$:
\begin{equation}\label{eq:trpo_imp}
\begin{aligned}
\rho_{\gamma}(\pi_{k+1}) - \rho_{\gamma}(\pi_k) &\geq  \frac{1}{1-\gamma} \E_{\substack{s\sim d_{\pik,\gamma} \\ a\sim\pi_{k+1}}}[A_{\gamma}^{\pik}(s,a)] \\
&- C\cdot\max_s[\TV{\pikup}{\pik}[s]] 
\end{aligned}
\end{equation}
where $\TV{\pi'}{\pi}[s]:=\frac{1}{2}\sum_a|\pi'(a|s)-\pi(a|s)|$ is the \textit{total variation divergence}, and $C=4\gamma\epsilon/(1-\gamma)^2$ where $\epsilon$ is some constant. \citet{schulman2015trust} showed that by choosing $\pikup$ which maximizes the right hand side of \eqref{eq:trpo_imp}, we are guaranteed to have $\rho_{\gamma}(\pikup)\geq\rho_{\gamma}(\pik)$. This provided the theoretical foundation for an entire class of on-policy DRL algorithms
\citep{schulman2015trust,schulman2017proximal, wu2017scalable,vuong2019supervised,song2020v}.

A natural question arises here is whether the iterative procedure described by \citet{schulman2015trust} also guarantees improvement for the average reward. Since the discounted and average reward objectives become equivalent as $\gamma\to 1$, one may conjecture that we can also lower bound the policy performance difference of the average reward objective by simply letting $\gamma\to 1$ for the bounds in \citet{schulman2015trust}.  Unfortunately this results in a non-meaningful bound (see supplementary material for proof.) 

\begin{restatable}{proposition}{uninformbound}\label{prop:uninform_bound}
Consider the bounds in Theorem 1 of \citet{schulman2015trust} and Corollary 1 of \citet{achiam2017constrained}. The right hand side of both bounds times $1-\gamma$ goes to negative infinity as $\gamma\to 1$.
\end{restatable}
Since $\lim_{\gamma\to 1}(1-\gamma)(\rho_{\gamma}(\pi') - \rho_{\gamma}(\pi)) = \rho(\pi') - \rho(\pi)$, Proposition \ref{prop:uninform_bound} says that the policy improvement guarantee from \citet{schulman2015trust} and \citet{achiam2017constrained} becomes trivial when $\gamma\to 1$ and thus does not generalize to the average reward setting. In the next section, we will derive a novel policy improvement bound for the average reward objective, which in turn can be used to generate monotonically improved policies w.r.t. the average reward.

\section{Main Results}

\subsection{Average Reward Policy Improvement Theorem}
 
Let $\dpi\in\R^{|\cS|}$ be the probability column vector whose components are $\dpi(s)$.
Let $P_{\pi}\in\R^{|\cS|\times|\cS|}$ be the transition matrix under policy $\pi$ whose $(s,s')$ component is $P_{\pi}(s'|s)=\sum_a P(s'|s,a)\pi(a|s)$, and $P_{\pi}^{\star}:=\lim_{N\to\infty}\frac{1}{N}\sum_{t=0}^N P_{\pi}^t$ be the limiting distribution of the transition matrix. For aperiodic unichain MDPs, $P_{\pi}^{\star} =\lim _{t\to\infty}P_{\pi}^t=\1\dpi^T$.

Suppose we have a new policy $\pi'$ obtained via some update rule from the current policy $\pi$. Similar to the discounted case, we would like to measure their performance difference $\rho(\pi')-\rho(\pi)$ using an expression which depends on $\pi$ and some divergence metric between the two policies.
The following identity shows that $\rho(\pi')-\rho(\pi)$ can be expressed using the average reward advantange function of $\pi$.
\begin{restatable}{lemma}{policydiff}\label{lemma:policy_diff}
Under Assumption \ref{assump:unichain}:
\begin{equation}\label{eq:policy_diff}
    \rho(\pi') - \rho(\pi) =  \E_{\substack{s\sim \dpip \\ a\sim\pi'}}\left[\adv(s,a)\right]
\end{equation}
for any two stochastic policies $\pi$ and $\pi'$.
\end{restatable}
Lemma \ref{lemma:policy_diff} is an extension of the well-known policy difference lemma from \citet{kakade2002approximately} to the average reward case. A similar result was proven by \citet{even2009online} and \citet{neu2010online}. For completeness, we provide a simple proof in the supplementary material. Note that this expression depends on samples drawn from $\pi'$. However we can show through the following lemma that when $\dpi$ and $\dpip$ are ``close'' w.r.t. the TV divergence, we can evaluate $\rho(\pi')$ using samples from $\dpi$ (see supplementary material for proof).

\begin{restatable}{lemma}{policyimpd}\label{lemma:policy_impd}
Under Assumption \ref{assump:unichain}, the following bound holds for any two stochastic policies $\pi$ and $\pi'$:
\begin{equation}
   \left| \rho(\pi') - \rho(\pi) -\E_{\substack{s\sim \dpi\\ a\sim\pi'}}\left[\adv(s,a)\right] \right| \leq 2\epsilon\TV{\dpip}{\dpi}
\end{equation}
where $\epsilon=\max_s\left|\E_{a\sim\pi'(a|s)}[\adv(s,a)]\right|$.
\end{restatable}
Lemma \ref{lemma:policy_impd} implies that
\begin{equation}
 \rho(\pi')\approx \rho(\pi) +\E_{\substack{s\sim \dpi\\ a\sim\pi'}}\left[\adv(s,a)\right]   
\end{equation}
when $\dpi$ and $\dpip$ are ``close''. However in order to study how policy improvement is connected to changes in the actual policies themselves, we need to analyze the relationship between changes in the policies and changes in stationary distributions. It turns out that the sensitivity of the stationary distributions in relation to the policies is related to the structure of the underlying Markov chain.

Let $\Mpi\in\R^{|\cS|\times|\cS|}$ be the \textit{mean first passage time matrix} whose elements $\Mpi(s,s')$ is the expected number of steps it takes to reach state $s'$ from $s$ under policy $\pi$. Under Assumption \ref{assump:irreducible}, the matrix $\Mpi$ can be calculated via (see Theorem 4.4.7 of \citet{kemeny1960finite})
\begin{equation}
    \Mpi = (I-\Zpi + E\Zpi_{\dg})D^{\pi}
\end{equation}
where $\Zpi = (I-P_{\pi}+P^{\star}_{\pi})^{-1}$ is known as the \textit{fundamental matrix of the Markov chain} \citep{kemeny1960finite}, $E$ is a square matrix consisting of all ones. The subscript `dg' on some square matrix refers to taking the diagonal of said matrix and placing zeros everywhere else. $D^{\pi}\in\R^{|\cS|\times|\cS|}$ is a diagonal matrix whose elements are $1/\dpi(s)$. 

One important property of mean first passage time is that for any MDP which satisfies Assumption \ref{assump:irreducible}, the quantity
\begin{equation}
    \kappa^{\pi} = \sum_{s'} \dpi(s') \Mpi(s,s') = \Tr(\Zpi)
\end{equation}
is a constant independent of the starting state for any policy $\pi$ (Theorem 4.4.10 of \citet{kemeny1960finite}.) The constant $\kappa^{\pi}$ is sometimes referred to as \textit{Kemeny's constant} \citep{grinstead2012introduction}. This constant can be interpreted as the mean number of steps it takes to get to any goal state 
weighted by the steady-distribution of the goal states. This weighted mean does not depend on the starting state, as mentioned just above. 

It can be shown that the value of Kemeny's constant is also related to the \emph{mixing time} of the Markov Chain, i.e., how fast the chain converges to the stationary distribution (see Appendix \ref{append:mixing} for additional details).

The following result connects the sensitivity of the stationary distribution to changes to the policy.
\begin{restatable}{lemma}{dandpi}\label{lemma:d_and_pi}
Under Assumption \ref{assump:irreducible}, the divergence between the stationary distributions $\dpi$ and $\dpip$ can be upper bounded by the average divergence between policies $\pi$ and $\pi'$:
\begin{equation}
    \TV{\dpip}{\dpi} \leq (\kappa^{\star}-1) \E_{s\sim \dpi}[\TV{\pi'}{\pi}[s]]
\end{equation}
where $\kappa^{\star} = \max_{\pi} \kappa^{\pi}$
\end{restatable}

For Markov chains with a small mixing time,
where an agent can quickly get to any state,
Kemeny's constant is relatively small and Lemma \ref{lemma:d_and_pi} shows that the stationary distributions are not highly sensitive to small changes in the policy. On the other hand, for Markov chains that that have high mixing times, the factor can become very large. In this case Lemma \ref{lemma:d_and_pi} shows that small changes in the policy can have a large impact on the resulting stationary distributions.

Combining the bounds in Lemma \ref{lemma:policy_impd} and Lemma \ref{lemma:d_and_pi} gives us the following result:
\begin{theorem}\label{thm:AvgR_policy_imp}
Under Assumption \ref{assump:irreducible} the following bounds hold for any two stochastic policies $\pi$ and $\pi'$, :
\begin{equation}\label{eq:AvgR_policy_imp}
 D_{\pi}^{-}(\pi')\leq \rho(\pi') - \rho(\pi)\leq D_{\pi}^{+}(\pi') 
\end{equation}
where
\begin{align*}
    D_{\pi}^{\pm}(\pi') = \E_{\substack{s\sim \dpi\\ a\sim\pi'}}\left[\adv(s,a)\right] \pm 2\xi\E_{s\sim \dpi}[\TV{\pi'}{\pi}[s]]
\end{align*}
and $\xi = (\kappa^{\star} - 1)\max_{s}\E_{a\sim\pi'}|\adv(s,a)|$.
\end{theorem}
The bounds in Theorem \ref{thm:AvgR_policy_imp} are guaranteed to be finite. Analogous to the discounted case, the multiplicative factor $\xi$ provides guidance on the step-sizes for policy updates. Note that Theorem \ref{assump:irreducible} holds for MDPs satisfying Assumption \ref{assump:irreducible}; in Appendix \ref{append:unichain_aperiodic} we discuss how a similar result can be derived for the more general aperiodic unichain case.

The bound in Theorem \ref{thm:AvgR_policy_imp} is given in terms of the TV divergence; however the KL divergence is more commonly used in practice. The relationship between the TV divergence and KL divergence is given by Pinsker's inequality \citep{tsybakov2008introduction}, which says that for any two distributions $p$ and $q$: $\TV{p}{q}\leq\sqrt{\KL{p}{q}/2}$. We can then show that
\begin{equation}\label{eq:tv-kl}
\begin{aligned}
  \E_{s\sim\dpi}[\TV{\pi'}{\pi}[s]] &\leq \E_{s\sim\dpi}[\sqrt{\KL{\pi'}{\pi}[s]/2}] \\
  &\leq \sqrt{\E_{s\sim\dpi}[\KL{\pi'}{\pi}][s]]/2}
\end{aligned}
\end{equation}
where the second inequality comes from Jensen's inequality. The inequality in \eqref{eq:tv-kl} shows that the bounds in Theorem \ref{thm:AvgR_policy_imp} still hold when $\E_{s\sim\dpi}[\TV{\pi'}{\pi}[s]]$ is  substituted with $\sqrt{\E_{s\sim\dpi}[\KL{\pi'}{\pi}][s]/2}$.

\subsection{Approximate Policy Iteration}
One direct consequence of Theorem \ref{thm:AvgR_policy_imp} is that iteratively maximizing the $D_{\pi}^{-}(\pi')$ term in the bound generates a monotonically improving sequence of policies w.r.t. the average reward objective. Algorithm \ref{alg:policy_iteration} gives an approximate policy iteration algorithm that produces such a sequence of policies.
\begin{algorithm}[tb]
  \caption{Approximate Average Reward Policy Iteration}
\label{alg:policy_iteration}
\begin{algorithmic}[1]
  \STATE {\bfseries Input:} $\pi_0$
  \FOR{$k=0,1,2,\dots$}
  \STATE Policy Evaluation: Evaluate $\advk(s,a)$ for all $s,a$
  \STATE Policy Improvement: 
  \begin{equation}\label{eq:politer_update}
      \pi_{k+1} = \argmax_{\pi} D_{\pik}^{-}(\pi)
  \end{equation}
  where
  \begin{align*}
  D_{\pik}^{-}(\pi) =& \E_{\substack{s\sim \dpik\\ a\sim\pi}}\left[\wb{A}^{\pik}(s,a)\right] \\
  &- \xi\sqrt{2\E_{s\sim\dpik }[\KL{\pi}{\pik}[s]]}
  \end{align*}
  and $\xi =(\kappa^{\star}-1) \max_{s}\E_{a\sim\pi}|\advk(s,a)|$
  \ENDFOR
\end{algorithmic}
\end{algorithm}
\begin{proposition}
Given an initial policy $\pi_0$, Algorithm \ref{alg:policy_iteration} is guaranteed to generate a sequence of policies $\pi_1,\pi_2,\dots$ such that $\rho(\pi_0)\leq\rho(\pi_1)\leq\rho(\pi_2)\leq\cdots$.
\end{proposition}
\begin{proof}
At iteration $k$, $\E_{s\sim\dpik, a\sim\pi}[\advk(s,a)]=0$, $\E_{s\sim\dpik}[\KL{\pi}{\pik}[s]]=0$ for $\pi=\pik$. By Theorem \ref{thm:AvgR_policy_imp} and \eqref{eq:politer_update}, $\rho(\pikup)-\rho(\pik)\geq 0$.
\end{proof}
However, Algorithm \ref{alg:policy_iteration} is difficult to implement in practice since it requires exact knowledge of $\advk(s,a)$ and the transition matrix. Furthermore, calculating the term $\xi$ is impractical for high-dimensional problems. In the next section, we will introduce a sample-based algorithm which approximates the update rule in Algorithm \ref{alg:policy_iteration}.

\section{Practical Algorithm}

As noted in the previous section, Algorithm \ref{alg:policy_iteration} is not practical for problems with large state and action spaces. In this section, we will discuss how Algorithm \ref{alg:policy_iteration} and Theorem \ref{thm:AvgR_policy_imp} can be used in practice to create algorithms which can effectively solve high dimensional DRL problems with the use of \emph{trust region} methods.

In Appendix \ref{append:acpo}, we will also discuss how Theorem \ref{thm:AvgR_policy_imp} can be used to solve DRL problems with average cost safety constraints. RL with safety constraints are an important class of problems with practical implications \citep{amodei2016concrete}. Trust region methods have been successfully applied to this class of problems as it provides worst-case constraint violation guarantees for evaluating the cost constraint values for policy updates \citep{achiam2017constrained,yang2020projection,zhang2020first}. However the aforementioned theoretical guarantees were only shown to apply to discounted cost constraints. \citet{tessler2018reward} pointed out that trust-region based methods such as the Constrained Policy Optimization (CPO) algorithm \citep{achiam2017constrained} cannot be used for average costs constraints. Contrary to this belief, in Appendix \ref{append:acpo}, we demonstrate that Theorem \ref{thm:AvgR_policy_imp} provides a worst-case constraint violation guarantee for average costs and trust-region-based constrained RL methods can easily be modified to accommodate for average cost constraints.

\subsection{Average Reward Trust Region Methods}\label{sec:atrpo}
For DRL problems, it is common to consider some parameterized policy class $\Pi_{\Theta} = \{\pi_{\theta}: \theta \in \Theta \}$. Our goal is to devise a computationally tractable version of Algorithm \ref{alg:policy_iteration} for policies in $\Pi_{\Theta}$. We can rewrite the unconstrained optimization problem in \eqref{eq:politer_update} as a constrained problem:
\begin{equation}\label{eq:atrpo}
\begin{aligned}
  & \underset{\pol\in\Pi_{\Theta}}{\text{maximize}} \quad
\E_{\substack{s\sim\dpolk\\ a\sim\pol}}[\wb{A}^{\polk}(s,a)] \\
& \text{subject to} \quad \avKL{\pol}{\polk} \leq \delta 
\end{aligned}
\end{equation}
where $\avKL{\pol}{\polk} := \E_{s\sim\dpolk}[\KL{\pol}{\polk}[s]]$. 
Importantly, the advantage function $\wb{A}^{\polk}(s,a)$ appearing in \eqref{eq:atrpo} is the average-reward advantage function, defined as the bias minus the action-bias, and not the discounted advantage function. The constraint set $\{\pol\in\Pi_{\Theta}:\avKL{\pol}{\polk}\leq \delta\}$ is called the \textit{trust region set}. The problem \eqref{eq:atrpo} can be regarded as an average reward variant of the trust region problem from \citet{schulman2015trust}. The step-size $\delta$ is treated as a hyperparamter in practice and should ideally be tuned for each specific task. However we note that in the average reward, the choice of step-size is related to the mixing time of the underlying Markov chain (since it is related to the multiplicative factor $\xi$ in Theorem \ref{thm:AvgR_policy_imp}). When the mixing time is small, a larger step-size can be chosen and vice versa. While it is impractical to calculate the optimal step-size, in certain applications domain knowledge on the mixing time can be used to serve as a guide for tuning $\delta$.

When we set $\polkup$ to be the optimal solution to \eqref{eq:atrpo}, similar to the discounted case, the policy improvement guarantee no longer holds. However we can show that $\polkup$ has the following worst-case performance degradation guarantee:
\begin{proposition}\label{prop:atrpo_perf}
Let $\polkup$ be the optimal solution to \eqref{eq:atrpo} for some $\polk\in\Pi_{\Theta}$. The policy performance difference between $\polkup$ and $\polk$ can be lower bounded by
\begin{equation}
    \rho(\polkup) - \rho(\polk) \geq -\xi^{\polkup}\sqrt{2\delta}
\end{equation}
where $\xi^{\polkup}=(\kappa^{\polkup}-1)\max_{s}\E_{a\sim\polkup}|\wb{A}^{\polk}(s,a)|$.
\end{proposition}
\begin{proof}
Since $\avKL{\polk}{\polk}=0$, $\polk$ is feasible. The objective value is 0 for $\pol=\polk$. The bound follows from \eqref{eq:AvgR_policy_imp} and \eqref{eq:tv-kl} where the average KL is bounded by $\delta$.
\end{proof}

Several algorithms have been proposed for efficiently solving the discounted version of \eqref{eq:atrpo}: \citet{schulman2015trust} and \citet{wu2017scalable}
converts \eqref{eq:atrpo} into a convex problem via Taylor approximations; another approach is to first solve $\eqref{eq:atrpo}$ in the non-parametric policy space and then project the result back into the parameter space \citep{vuong2019supervised,song2020v}. These algorithms can also be adapted for the average reward case and are theoretically justified via Theorem \ref{thm:AvgR_policy_imp} and Proposition \ref{prop:atrpo_perf}. In the next section, we will provide as a specific example how this can be done for one such algorithm.

\subsection{Average Reward TRPO (ATRPO)}
In this section, we introduce ATRPO, which is an average-reward modification of the TRPO algorithm \citep{schulman2015trust}. Similar to TRPO, we apply Taylor approximations to \eqref{eq:atrpo}. This gives us a new optimization problem which can be solved exactly using Lagrange duality \citep{boyd2004convex}. The solution to this approximate problem gives an explicit update rule for the policy parameters which then allows us to perform policy updates using an actor-critic framework. More details can be found in Appendix \ref{append:atrpo}. Algorithm \ref{alg:atrpo} provides a basic outline of ATRPO.

\begin{algorithm}[tb]
  \caption{Average Reward TRPO (ATRPO)}
\label{alg:atrpo}
\begin{algorithmic}[1]
  \STATE {\bfseries Input:} Policy parameters $\theta_0$, critic net parameters $\phi_0$, learning rate $\alpha$, trajectory truncation parameter $N$.
  \FOR{$k=0,1,2,\cdots$}
  \STATE Collect a truncated trajectory $\{s_t, a_t, s_{t+1}, r_t\},\; t=1,\dots,N$ from the environment using $\polk$.
  \STATE Calculate sample average reward of $\polk$ via \\
  $\rho = \frac{1}{N}\sum_{t=1}^{N} r_t$.
  \FOR{$t=1,2,\dots,N$}
  \STATE Get target $\vtarg_t = r_t - \rho + \wb{V}_{\phi_k}(s_{t+1})$
  \STATE Get advantage estimate:\\ $\hat{A}(s_t,a_t) = r_t - \rho + \wb{V}_{\phi_k}(s_{t+1}) - \wb{V}_{\phi_k}(s_{t}) $
  \ENDFOR
     \STATE Update critic by
    \[
    \phi_{k+1} \gets \phi_k - \alpha\nabla_{\phi}\cL(\phi_k)
    \]
    where
    \[
    \cL(\phi_k) = \frac{1}{N}\sum_{t=1}^{N}\norm{\bar{V}_{\phi_k}(s_{t})-\vtarg_t}^2
    \]
    \STATE Use $\hat{A}(s_t,a_t)$ to update $\theta_k$ using TRPO policy update \citep{schulman2015trust}.
  \ENDFOR
\end{algorithmic}
\end{algorithm}

The major differences between ATRPO and TRPO are as follows: 
\begin{enumerate}[i]
    \item The critic network in Algorithm \ref{alg:atrpo} approximates the average-reward bias rather than the discounted value function.
    \item ATRPO must estimate the average return $\rho$ of the current policy.
    \item The targets for the bias and the advantage are calculated without discount factors and the average return $\rho$ is subtracted from the reward. Simply setting the discount factor to 1 in TRPO does not lead to Algorithm \ref{alg:atrpo}. 
    \item ATRPO also assumes that the underlying task is a continuing infinite-horizon task. But since in practice we cannot run infinitely long trajectories, all trajectories are truncated at some large truncation value $N$. Unlike TRPO, during training we do not allow for episodic tasks where episodes terminate early (before $N$). For the MuJoCo environments, we will address this by having the agent not only resume locomotion after falling  but also incur a penalty for falling (see Section \ref{sec:experiments}.) 
\end{enumerate}
In Algorithm 2, for illustrative purposes, we use the average reward one-step bootstrapped estimate for the target of the critic and the advantage function. In practice, we instead develop and use an average-reward version of the Generalized Advantage Estimator (GAE) from \citet{schulman2016high}. In Appendix \ref{append:GAE} we provide more details on how GAE can be generalized to the average-reward case.

\begin{figure*}[ht]
    \centering
    \includegraphics[width=0.9\textwidth]{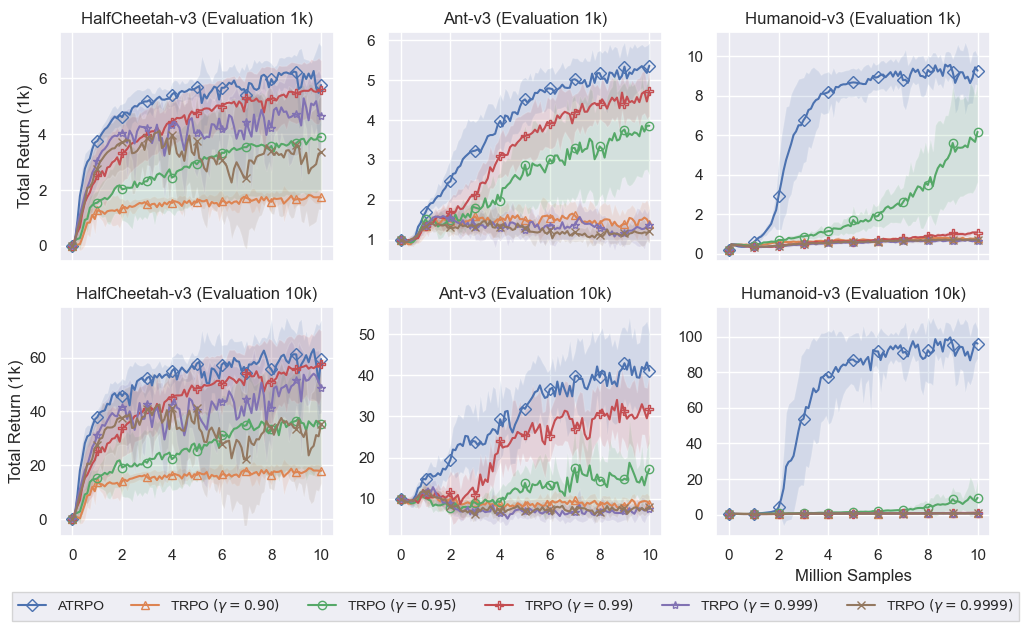}
    \caption{Comparing performance of ATRPO and TRPO with different discount factors. The $x$-axis is the number of agent-environment interactions and the $y$-axis is the total return averaged over 10 seeds. The solid line represents the agents' performance on evaluation trajectories of maximum length 1,000 (top row) and 10,000 (bottom row). The shaded region represents one standard deviation.}
    \label{fig:trpo}
\end{figure*}

\section{Experiments}\label{sec:experiments}

We conducted experiments comparing the performance of ATRPO and TRPO on continuing control tasks. We consider three tasks (Ant, HalfCheetah, and Humanoid) from the MuJoCo physical simulator \citep{todorov2012mujoco} implemented using OpenAI gym \citep{brockman2016openai}, where the natural goal is to train the agents to run as fast as possible without falling.

\subsection{Evaluation Protocol}\label{sec:eval}

Even though the MuJoCo benchmark is commonly trained using the \emph{discounted} objective (see e.g. \citet{schulman2015trust}, \citet{wu2017scalable},
\citet{lillicrap2016continuous},
\citet{schulman2017proximal},
\citet{haarnoja2018soft}, \citet{vuong2019supervised}), it is {\em always} evaluated without discounting. Similarly, we also evaluate performance
using the undiscounted total-reward objective for both TRPO and ATRPO.

Specifically for each environment, we train a policy for 10 million environment steps. During training, every 100,000 steps, we run 10 separate evaluation trajectories with the current policy without exploration (i.e., the policy is kept fixed and deterministic). For each evaluation trajectory we calculate the undiscounted return of the trajectory until the agent falls or until 1,000 steps, whichever comes first. We then report the average undiscounted return over the 10 trajectories. \emph{Note that this is the standard evaluation metric for the MuJoCo environments.}
In order to understand the performance of the agent for long time horizons, we also report the performance of the agent evaluated on trajectories of maximum length 10,000. 

\begin{figure*}[ht]
\centering
\includegraphics[width=\textwidth]{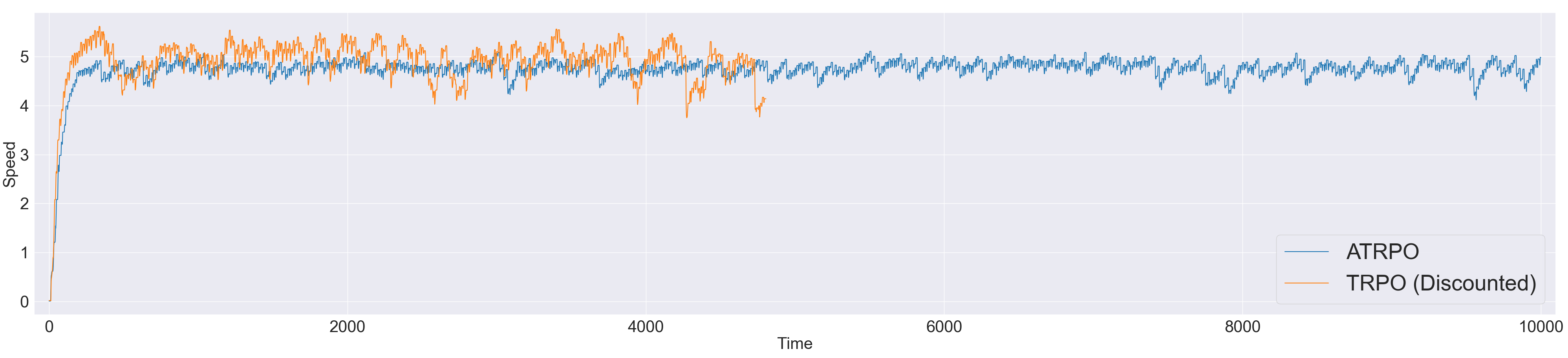}
\caption{Speed-time plot of a single trajectory (maximum length 10,000) for ATRPO and Discounted TRPO in the Humanoid-v3 environment. The solid line represents the speed of the agent at the corresponding timesteps.}
\label{fig:trpo_reset_traj}
\end{figure*}

\subsection{Comparing ATRPO and TRPO}

To simulate an infinite-horizon setting during training, we do the following: when the agent falls, the trajectory does not terminate; instead the agent incurs a large reset cost for falling, and then continues the trajectory from a random start state. The reset cost is set to 100. However, we show in the supplementary material (Appendix \ref{append:sensitivity}) that the results are largely insensitive to the choice of reset cost. We note that this modification does not change the underlying goal of the task. 
We also point out that the reset cost is only applied during training and is not used in the evaluation phase described in the previous section. Hyperparameter settings and other additional details can be found in Appendix \ref{append:experiment}.

We plot the performance for ATRPO and TRPO trained with different discount factors in Figure \ref{fig:trpo}. We see that TRPO with its best discount factor can perform as well as ATRPO for the simplest environment HalfCheetah.
But ATRPO provides dramatic improvements in Ant and Humanoid.
In particular for the most challenging environment Humanoid, ATRPO performs on average $50.1\%$ better than TRPO with its best discount factor when evaluated on trajectories of maximum length 1000. The improvement is even greater when the agents are evaluated on trajectories of maximum length 10,000 where the performance boost jumps to $913\%$. 
In Appendix \ref{append:trpo_mod}, we provide an additional set of experiments demonstrating that ATRPO also significantly outperforms TRPO when TRPO is trained without the reset scheme described at the beginning of this section (i.e. the standard MuJoCo setting.) 

We make two observations regarding discounting. First, we note that increasing the discount factor does not necessarily lead to better performance for TRPO. A larger discount factor in principle enables the algorithm to seek a policy that performs well for the average-reward criterion \citep{blackwell1962discrete}. Unfortunately, a larger discount factor can also increase the variance of the gradient estimator \citep{zhao2011analysis,schulman2016high}, increase the complexity of the policy space \citep{jiang2015dependence}, lead to slower convergence \citep{bertsekas1995dynamic,agarwal2020optimality}, and degrade generalization in limited data settings \citep{amit2020discount}.
Moreover, algorithms with discounting are known to become unstable as $\gamma\to 1$ \citep{naik2019discounted}. Secondly, for TRPO the best discount factor is different for each environment (0.99 for HalfCheetah and Ant, 0.95 for Humanoid). The discount factor therefore serves as a hyperparameter which can be tuned to improve performance, choosing a suboptimal discount factor can have significant consequences. Both of these observation are consistent with what was seen in the literature \citep{andrychowicz2020matters}. We have shown here that using the average reward criterion directly not only delivers superior performance but also obviates the need to tune the discount factor.

\subsection{Understanding Long Run Performance}

Next, we demonstrate that agents trained using the average reward criterion are better at optimizing for long-term returns. Here, we first train Humanoid with 10 million samples with ATRPO and with TRPO with a discount factor of 0.95 (shown to be the best discount factor in the previous experiments). 
Then for evaluation, we run the trained ATRPO and TRPO policies for a trajectory of 10,000 timesteps (or until the agent falls). We use the same random seeds for the two algorithms. Figure \ref{fig:trpo_reset_traj} is a plot of the speed of the agent at each time step of the trajectory, using the \emph{seed that gives the best performance for discounted TRPO}. We see in Figure \ref{fig:trpo_reset_traj} that the discounted algorithm gives a higher initial speed at the beginning of the trajectory. However its overall speed is much more erratic throughout the trajectory, resulting in the agent falling over after approximately 5000 steps. This coincides with the notion of discounting where more emphasis is placed at the beginning of the trajectory and ignores longer-term behavior. On the other hand,  the average-reward policy \textemdash\, while having a slightly lower velocity overall throughout its trajectory \textemdash\, is able to sustain the trajectory much longer, thus giving it a higher total return. In fact, we observed that for all 10 random seeds we tested, the average reward agent is able to finish the entire 10,000 time step trajectory without falling. In Table \ref{tab:traj_summary} we present the summary statistics of trajectory length for all trajectories using discounted TRPO we note that the median trajectory length for the TRPO discounted agent is 452.5, meaning that on average TRPO performs significantly worse than what is reported in Figure. \ref{fig:trpo_reset_traj}. 
\begin{table}[h]
    \centering
    \caption{Summary statistics for all 10 trajectories using a Humanoid-v3 agent trained with TRPO}
    \vskip 0.15in
    \begin{tabular}{*5c}
    \toprule
     Min & Max & Average & Median & Std \\
        \midrule
    108 & 4806 & 883.1 & 452.5 & 1329.902 \\
       \bottomrule
    \end{tabular}
    \label{tab:traj_summary}
\end{table}

\section{Related Work}
Dynamic programming algorithms for finding the optimal average reward policies have been well-studied \citep{howard1960dynamic,blackwell1962discrete,veinott1966finding}. Several tabular Q-learning-like algorithms for problems with unknown dynamics have been proposed, such as R-Learning \citep{schwartz1993reinforcement}, RVI Q-Learning \citep{abounadi2001learning},  CSV-Learning \citep{yang2016efficient}, and Differential Q-Learning \citep{wan2020learning}. \citet{mahadevan1996average} conducted a thorough empirical analysis of the R-Learning algorithm. We note that much of the previous work on average reward RL focuses on the tabular setting without function approximations, and the theoretical properties of many of these Q-learning-based algorithm are not well understood (in particular R-learning). More recently, POLITEX updates policies using a Boltzmann distribution over the
sum of action-value function estimates of the previous policies \citep{abbasi2019politex} and \citet{wei2020model} introduced a model-free algorithm for optimizing the average reward of weakly-communicating MDPs. 

For policy gradient methods, \citet{baxter2001infinite} showed that if $1/(1-\gamma)$ is large compared to the mixing time of the Markov chain induced by the MDP, then the gradient of $\rho_{\gamma}(\pi)$ can accurately approximate the gradient of $\rho(\pi)$. \citet{kakade2001optimizing} extended upon this result and provided an error bound on using an optimal discounted policy to maximize the average reward. In contrast, our work directly deals with the average reward objective and provides theoretical guidance on the optimal step size for each policy update.

Policy improvement bounds have been extensively explored in the discounted case. The results from \citet{schulman2015trust} are extensions of \citet{kakade2002approximately}. 
\citet{pirotta2013safe} also proposed an alternative generalization to \citet{kakade2002approximately}.  \citet{achiam2017constrained} improved upon \citet{schulman2015trust} by replacing the maximum divergence with the average divergence.

\section{Conclusion}
In this paper, we introduce a novel policy improvement bound for the average reward criterion. The bound is based on the average divergence between two policies and Kemeny's constant or mixing time of the Markov chain. We show that previous existing policy improvement bounds for the discounted case results in a non-meaningful bound for the average reward objective. Our work provides the theoretical justification and the means to generalize the popular trust-region based algorithms to the average reward setting. Based on this theory, we propose ATRPO, a modification of the TRPO algorithm for on-policy DRL. We demonstrate through a series of experiments that ATRPO is highly effective on high-dimensional continuing control tasks. 

\section*{Acknowledgements}
We would like to extend our gratitude to Quan Vuong and the anonymous reviewers for their constructive comments and suggestions. We also thank Shuyang Ling, Che Wang, Zining (Lily) Wang, and Yanqiu Wu for the insightful discussions on this work.

\newpage

\bibliography{example_paper}

\begin{thebibliography}{62}
\providecommand{\natexlab}[1]{#1}
\providecommand{\url}[1]{\texttt{#1}}
\expandafter\ifx\csname urlstyle\endcsname\relax
  \providecommand{\doi}[1]{doi: #1}\else
  \providecommand{\doi}{doi: \begingroup \urlstyle{rm}\Url}\fi

\bibitem[Abbasi-Yadkori et~al.(2019)Abbasi-Yadkori, Bartlett, Bhatia, Lazic,
  Szepesvari, and Weisz]{abbasi2019politex}
Abbasi-Yadkori, Y., Bartlett, P., Bhatia, K., Lazic, N., Szepesvari, C., and
  Weisz, G.
\newblock Politex: Regret bounds for policy iteration using expert prediction.
\newblock In \emph{International Conference on Machine Learning}, pp.\
  3692--3702, 2019.

\bibitem[Abounadi et~al.(2001)Abounadi, Bertsekas, and
  Borkar]{abounadi2001learning}
Abounadi, J., Bertsekas, D., and Borkar, V.~S.
\newblock Learning algorithms for markov decision processes with average cost.
\newblock \emph{SIAM Journal on Control and Optimization}, 40\penalty0
  (3):\penalty0 681--698, 2001.

\bibitem[Achiam et~al.(2017)Achiam, Held, Tamar, and
  Abbeel]{achiam2017constrained}
Achiam, J., Held, D., Tamar, A., and Abbeel, P.
\newblock Constrained policy optimization.
\newblock In \emph{Proceedings of the 34th International Conference on Machine
  Learning-Volume 70}, pp.\  22--31. JMLR. org, 2017.

\bibitem[Agarwal et~al.(2020)Agarwal, Kakade, Lee, and
  Mahajan]{agarwal2020optimality}
Agarwal, A., Kakade, S.~M., Lee, J.~D., and Mahajan, G.
\newblock Optimality and approximation with policy gradient methods in markov
  decision processes.
\newblock In \emph{Conference on Learning Theory}, pp.\  64--66. PMLR, 2020.

\bibitem[Altman(1999)]{altman1999constrained}
Altman, E.
\newblock \emph{Constrained Markov decision processes}, volume~7.
\newblock CRC Press, 1999.

\bibitem[Amit et~al.(2020)Amit, Meir, and Ciosek]{amit2020discount}
Amit, R., Meir, R., and Ciosek, K.
\newblock Discount factor as a regularizer in reinforcement learning.
\newblock In \emph{International conference on machine learning}, 2020.

\bibitem[Amodei et~al.(2016)Amodei, Olah, Steinhardt, Christiano, Schulman, and
  Man{\'e}]{amodei2016concrete}
Amodei, D., Olah, C., Steinhardt, J., Christiano, P., Schulman, J., and
  Man{\'e}, D.
\newblock Concrete problems in ai safety.
\newblock \emph{arXiv preprint arXiv:1606.06565}, 2016.

\bibitem[Andrychowicz et~al.(2020)Andrychowicz, Raichuk, Sta{\'n}czyk, Orsini,
  Girgin, Marinier, Hussenot, Geist, Pietquin, Michalski,
  et~al.]{andrychowicz2020matters}
Andrychowicz, M., Raichuk, A., Sta{\'n}czyk, P., Orsini, M., Girgin, S.,
  Marinier, R., Hussenot, L., Geist, M., Pietquin, O., Michalski, M., et~al.
\newblock What matters in on-policy reinforcement learning? a large-scale
  empirical study.
\newblock \emph{arXiv preprint arXiv:2006.05990}, 2020.

\bibitem[Baxter \& Bartlett(2001)Baxter and Bartlett]{baxter2001infinite}
Baxter, J. and Bartlett, P.~L.
\newblock Infinite-horizon policy-gradient estimation.
\newblock \emph{Journal of Artificial Intelligence Research}, 15:\penalty0
  319--350, 2001.

\bibitem[Bertsekas et~al.(1995)Bertsekas, Bertsekas, Bertsekas, and
  Bertsekas]{bertsekas1995dynamic}
Bertsekas, D.~P., Bertsekas, D.~P., Bertsekas, D.~P., and Bertsekas, D.~P.
\newblock \emph{Dynamic programming and optimal control}, volume 1,2.
\newblock Athena scientific Belmont, MA, 1995.

\bibitem[Blackwell(1962)]{blackwell1962discrete}
Blackwell, D.
\newblock Discrete dynamic programming.
\newblock \emph{The Annals of Mathematical Statistics}, pp.\  719--726, 1962.

\bibitem[Boyd et~al.(2004)Boyd, Boyd, and Vandenberghe]{boyd2004convex}
Boyd, S., Boyd, S.~P., and Vandenberghe, L.
\newblock \emph{Convex optimization}.
\newblock Cambridge university press, 2004.

\bibitem[Br{\'e}maud(2020)]{bremaud2020markov}
Br{\'e}maud, P.
\newblock \emph{Markov Chains Gibbs Fields, Monte Carlo Simulation and Queues}.
\newblock Springer, 2 edition, 2020.

\bibitem[Brockman et~al.(2016)Brockman, Cheung, Pettersson, Schneider,
  Schulman, Tang, and Zaremba]{brockman2016openai}
Brockman, G., Cheung, V., Pettersson, L., Schneider, J., Schulman, J., Tang,
  J., and Zaremba, W.
\newblock Openai gym, 2016.

\bibitem[Cho \& Meyer(2001)Cho and Meyer]{cho2001comparison}
Cho, G.~E. and Meyer, C.~D.
\newblock Comparison of perturbation bounds for the stationary distribution of
  a markov chain.
\newblock \emph{Linear Algebra and its Applications}, 335\penalty0
  (1-3):\penalty0 137--150, 2001.

\bibitem[Even-Dar et~al.(2009)Even-Dar, Kakade, and Mansour]{even2009online}
Even-Dar, E., Kakade, S.~M., and Mansour, Y.
\newblock Online markov decision processes.
\newblock \emph{Mathematics of Operations Research}, 34\penalty0 (3):\penalty0
  726--736, 2009.

\bibitem[Grinstead \& Snell(2012)Grinstead and
  Snell]{grinstead2012introduction}
Grinstead, C.~M. and Snell, J.~L.
\newblock \emph{Introduction to probability}.
\newblock American Mathematical Soc., 2012.

\bibitem[Haarnoja et~al.(2018)Haarnoja, Zhou, Abbeel, and
  Levine]{haarnoja2018soft}
Haarnoja, T., Zhou, A., Abbeel, P., and Levine, S.
\newblock Soft actor-critic: Off-policy maximum entropy deep reinforcement
  learning with a stochastic actor.
\newblock \emph{International Conference on Machine Learning (ICML)}, 2018.

\bibitem[Horn \& Johnson(2012)Horn and Johnson]{horn2012matrix}
Horn, R.~A. and Johnson, C.~R.
\newblock \emph{Matrix analysis}.
\newblock Cambridge university press, 2012.

\bibitem[Howard(1960)]{howard1960dynamic}
Howard, R.~A.
\newblock \emph{Dynamic programming and markov processes.}
\newblock John Wiley, 1960.

\bibitem[Hunter(2005)]{hunter2005stationary}
Hunter, J.~J.
\newblock Stationary distributions and mean first passage times of perturbed
  markov chains.
\newblock \emph{Linear Algebra and its Applications}, 410:\penalty0 217--243,
  2005.

\bibitem[Jiang et~al.(2015)Jiang, Kulesza, Singh, and
  Lewis]{jiang2015dependence}
Jiang, N., Kulesza, A., Singh, S., and Lewis, R.
\newblock The dependence of effective planning horizon on model accuracy.
\newblock In \emph{Proceedings of the 2015 International Conference on
  Autonomous Agents and Multiagent Systems}, pp.\  1181--1189. Citeseer, 2015.

\bibitem[Jiang et~al.(2016)Jiang, Singh, and Tewari]{jiang2016structural}
Jiang, N., Singh, S.~P., and Tewari, A.
\newblock On structural properties of mdps that bound loss due to shallow
  planning.
\newblock In \emph{IJCAI}, pp.\  1640--1647, 2016.

\bibitem[Kakade(2001{\natexlab{a}})]{kakade2001optimizing}
Kakade, S.
\newblock Optimizing average reward using discounted rewards.
\newblock In \emph{International Conference on Computational Learning Theory},
  pp.\  605--615. Springer, 2001{\natexlab{a}}.

\bibitem[Kakade \& Langford(2002)Kakade and Langford]{kakade2002approximately}
Kakade, S. and Langford, J.
\newblock Approximately optimal approximate reinforcement learning.
\newblock In \emph{International Conference on Machine Learning}, volume~2,
  pp.\  267--274, 2002.

\bibitem[Kakade(2001{\natexlab{b}})]{kakade2001natural}
Kakade, S.~M.
\newblock A natural policy gradient.
\newblock \emph{Advances in neural information processing systems}, 14,
  2001{\natexlab{b}}.

\bibitem[Kallenberg(1983)]{kallenberg1983linear}
Kallenberg, L.
\newblock \emph{Linear Programming and Finite Markovian Control Problems}.
\newblock Centrum Voor Wiskunde en Informatica, 1983.

\bibitem[Kemeny \& Snell(1960)Kemeny and Snell]{kemeny1960finite}
Kemeny, J. and Snell, I.
\newblock \emph{Finite {M}arkov {C}hains}.
\newblock Van Nostrand, New Jersey, 1960.

\bibitem[Lehmann \& Casella(2006)Lehmann and Casella]{lehmann2006theory}
Lehmann, E.~L. and Casella, G.
\newblock \emph{Theory of point estimation}.
\newblock Springer Science \& Business Media, 2006.

\bibitem[Lehnert et~al.(2018)Lehnert, Laroche, and van
  Seijen]{lehnert2018value}
Lehnert, L., Laroche, R., and van Seijen, H.
\newblock On value function representation of long horizon problems.
\newblock In \emph{Proceedings of the AAAI Conference on Artificial
  Intelligence}, volume~32, 2018.

\bibitem[Levin \& Peres(2017)Levin and Peres]{levin2017markov}
Levin, D.~A. and Peres, Y.
\newblock \emph{Markov chains and mixing times}, volume 107.
\newblock American Mathematical Soc., 2017.

\bibitem[Lillicrap et~al.(2016)Lillicrap, Hunt, Pritzel, Heess, Erez, Tassa,
  Silver, and Wierstra]{lillicrap2016continuous}
Lillicrap, T.~P., Hunt, J.~J., Pritzel, A., Heess, N., Erez, T., Tassa, Y.,
  Silver, D., and Wierstra, D.
\newblock Continuous control with deep reinforcement learning.
\newblock \emph{International Conference on Learning Representations (ICLR)},
  2016.

\bibitem[Mahadevan(1996)]{mahadevan1996average}
Mahadevan, S.
\newblock Average reward reinforcement learning: Foundations, algorithms, and
  empirical results.
\newblock \emph{Machine learning}, 22\penalty0 (1-3):\penalty0 159--195, 1996.

\bibitem[Mnih et~al.(2013)Mnih, Kavukcuoglu, Silver, Graves, Antonoglou,
  Wierstra, and Riedmiller]{mnih2013playing}
Mnih, V., Kavukcuoglu, K., Silver, D., Graves, A., Antonoglou, I., Wierstra,
  D., and Riedmiller, M.
\newblock Playing atari with deep reinforcement learning.
\newblock \emph{NIPS Deep Learning Workshop}, 2013.

\bibitem[Naik et~al.(2019)Naik, Shariff, Yasui, and Sutton]{naik2019discounted}
Naik, A., Shariff, R., Yasui, N., and Sutton, R.~S.
\newblock Discounted reinforcement learning is not an optimization problem.
\newblock \emph{NeurIPS Optimization Foundations for Reinforcement Learning
  Workshop}, 2019.

\bibitem[Neu et~al.(2010)Neu, Antos, Gy{\"o}rgy, and
  Szepesv{\'a}ri]{neu2010online}
Neu, G., Antos, A., Gy{\"o}rgy, A., and Szepesv{\'a}ri, C.
\newblock Online markov decision processes under bandit feedback.
\newblock In \emph{Advances in Neural Information Processing Systems}, pp.\
  1804--1812, 2010.

\bibitem[Peters \& Schaal(2008)Peters and Schaal]{peters2008reinforcement}
Peters, J. and Schaal, S.
\newblock Reinforcement learning of motor skills with policy gradients.
\newblock \emph{Neural networks}, 21\penalty0 (4):\penalty0 682--697, 2008.

\bibitem[Petrik \& Scherrer(2008)Petrik and Scherrer]{petrik2008biasing}
Petrik, M. and Scherrer, B.
\newblock Biasing approximate dynamic programming with a lower discount factor.
\newblock In \emph{Twenty-Second Annual Conference on Neural Information
  Processing Systems-NIPS 2008}, 2008.

\bibitem[Pirotta et~al.(2013)Pirotta, Restelli, Pecorino, and
  Calandriello]{pirotta2013safe}
Pirotta, M., Restelli, M., Pecorino, A., and Calandriello, D.
\newblock Safe policy iteration.
\newblock In \emph{International Conference on Machine Learning}, pp.\
  307--315, 2013.

\bibitem[Ross(1985)]{ross1985constrained}
Ross, K.~W.
\newblock Constrained markov decision processes with queueing applications.
\newblock \emph{Dissertation Abstracts International Part B: Science and
  Engineering[DISS. ABST. INT. PT. B- SCI. \& ENG.],}, 46\penalty0 (4), 1985.

\bibitem[Schulman et~al.(2015)Schulman, Levine, Abbeel, Jordan, and
  Moritz]{schulman2015trust}
Schulman, J., Levine, S., Abbeel, P., Jordan, M., and Moritz, P.
\newblock Trust region policy optimization.
\newblock In \emph{International Conference on Machine Learning}, pp.\
  1889--1897, 2015.

\bibitem[Schulman et~al.(2016)Schulman, Moritz, Levine, Jordan, and
  Abbeel]{schulman2016high}
Schulman, J., Moritz, P., Levine, S., Jordan, M., and Abbeel, P.
\newblock High-dimensional continuous control using generalized advantage
  estimation.
\newblock \emph{International Conference on Learning Representations (ICLR)},
  2016.

\bibitem[Schulman et~al.(2017)Schulman, Wolski, Dhariwal, Radford, and
  Klimov]{schulman2017proximal}
Schulman, J., Wolski, F., Dhariwal, P., Radford, A., and Klimov, O.
\newblock Proximal policy optimization algorithms.
\newblock \emph{arXiv preprint arXiv:1707.06347}, 2017.

\bibitem[Schwartz(1993)]{schwartz1993reinforcement}
Schwartz, A.
\newblock A reinforcement learning method for maximizing undiscounted rewards.
\newblock In \emph{Proceedings of the tenth international conference on machine
  learning}, volume 298, pp.\  298--305, 1993.

\bibitem[Silver et~al.(2016)Silver, Huang, Maddison, Guez, Sifre, Van
  Den~Driessche, Schrittwieser, Antonoglou, Panneershelvam, Lanctot,
  et~al.]{silver2016mastering}
Silver, D., Huang, A., Maddison, C.~J., Guez, A., Sifre, L., Van Den~Driessche,
  G., Schrittwieser, J., Antonoglou, I., Panneershelvam, V., Lanctot, M.,
  et~al.
\newblock Mastering the game of go with deep neural networks and tree search.
\newblock \emph{nature}, 529\penalty0 (7587):\penalty0 484, 2016.

\bibitem[Silver et~al.(2018)Silver, Hubert, Schrittwieser, Antonoglou, Lai,
  Guez, Lanctot, Sifre, Kumaran, Graepel, et~al.]{silver2018general}
Silver, D., Hubert, T., Schrittwieser, J., Antonoglou, I., Lai, M., Guez, A.,
  Lanctot, M., Sifre, L., Kumaran, D., Graepel, T., et~al.
\newblock A general reinforcement learning algorithm that masters chess, shogi,
  and go through self-play.
\newblock \emph{Science}, 362\penalty0 (6419):\penalty0 1140--1144, 2018.

\bibitem[Song et~al.(2020)Song, Abdolmaleki, Springenberg, Clark, Soyer, Rae,
  Noury, Ahuja, Liu, Tirumala, et~al.]{song2020v}
Song, H.~F., Abdolmaleki, A., Springenberg, J.~T., Clark, A., Soyer, H., Rae,
  J.~W., Noury, S., Ahuja, A., Liu, S., Tirumala, D., et~al.
\newblock V-mpo: on-policy maximum a posteriori policy optimization for
  discrete and continuous control.
\newblock \emph{International Conference on Learning Representations}, 2020.

\bibitem[Sutton \& Barto(2018)Sutton and Barto]{sutton2018reinforcement}
Sutton, R.~S. and Barto, A.~G.
\newblock \emph{Reinforcement learning: An introduction}.
\newblock MIT press, 2018.

\bibitem[Sutton et~al.(2000)Sutton, McAllester, Singh, and
  Mansour]{sutton2000policy}
Sutton, R.~S., McAllester, D.~A., Singh, S.~P., and Mansour, Y.
\newblock Policy gradient methods for reinforcement learning with function
  approximation.
\newblock In \emph{Advances in neural information processing systems}, pp.\
  1057--1063, 2000.

\bibitem[Tadepalli \& Ok(1994)Tadepalli and Ok]{tadepalli1994h}
Tadepalli, P. and Ok, D.
\newblock H-learning: A reinforcement learning method to optimize undiscounted
  average reward.
\newblock Technical Report 94-30-01, Oregon State University, 1994.

\bibitem[Tessler et~al.(2019)Tessler, Mankowitz, and Mannor]{tessler2018reward}
Tessler, C., Mankowitz, D.~J., and Mannor, S.
\newblock Reward constrained policy optimization.
\newblock \emph{International Conference on Learning Representation (ICLR)},
  2019.

\bibitem[Todorov et~al.(2012)Todorov, Erez, and Tassa]{todorov2012mujoco}
Todorov, E., Erez, T., and Tassa, Y.
\newblock Mujoco: A physics engine for model-based control.
\newblock In \emph{2012 IEEE/RSJ International Conference on Intelligent Robots
  and Systems}, pp.\  5026--5033. IEEE, 2012.

\bibitem[Tsybakov(2008)]{tsybakov2008introduction}
Tsybakov, A.~B.
\newblock \emph{Introduction to nonparametric estimation}.
\newblock Springer Science \& Business Media, 2008.

\bibitem[Veinott(1966)]{veinott1966finding}
Veinott, A.~F.
\newblock On finding optimal policies in discrete dynamic programming with no
  discounting.
\newblock \emph{The Annals of Mathematical Statistics}, 37\penalty0
  (5):\penalty0 1284--1294, 1966.

\bibitem[Vuong et~al.(2019)Vuong, Zhang, and Ross]{vuong2019supervised}
Vuong, Q., Zhang, Y., and Ross, K.~W.
\newblock Supervised policy update for deep reinforcement learning.
\newblock In \emph{International Conference on Learning Representation (ICLR)},
  2019.

\bibitem[Wan et~al.(2020)Wan, Naik, and Sutton]{wan2020learning}
Wan, Y., Naik, A., and Sutton, R.~S.
\newblock Learning and planning in average-reward markov decision processes.
\newblock \emph{arXiv preprint arXiv:2006.16318}, 2020.

\bibitem[Wei et~al.(2020)Wei, Jafarnia-Jahromi, Luo, Sharma, and
  Jain]{wei2020model}
Wei, C.-Y., Jafarnia-Jahromi, M., Luo, H., Sharma, H., and Jain, R.
\newblock Model-free reinforcement learning in infinite-horizon average-reward
  markov decision processes.
\newblock In \emph{International conference on machine learning}, 2020.

\bibitem[Wu et~al.(2017)Wu, Mansimov, Grosse, Liao, and Ba]{wu2017scalable}
Wu, Y., Mansimov, E., Grosse, R.~B., Liao, S., and Ba, J.
\newblock Scalable trust-region method for deep reinforcement learning using
  kronecker-factored approximation.
\newblock In \emph{Advances in neural information processing systems (NIPS)},
  pp.\  5285--5294, 2017.

\bibitem[Yang et~al.(2016)Yang, Gao, An, Wang, and Chen]{yang2016efficient}
Yang, S., Gao, Y., An, B., Wang, H., and Chen, X.
\newblock Efficient average reward reinforcement learning using constant
  shifting values.
\newblock In \emph{AAAI}, pp.\  2258--2264, 2016.

\bibitem[Yang et~al.(2020)Yang, Rosca, Narasimhan, and
  Ramadge]{yang2020projection}
Yang, T.-Y., Rosca, J., Narasimhan, K., and Ramadge, P.~J.
\newblock Projection-based constrained policy optimization.
\newblock In \emph{International Conference on Learning Representation (ICLR)},
  2020.

\bibitem[Zhang et~al.(2020)Zhang, Vuong, and Ross]{zhang2020first}
Zhang, Y., Vuong, Q., and Ross, K.
\newblock First order constrained optimization in policy space.
\newblock \emph{Advances in Neural Information Processing Systems}, 33, 2020.

\bibitem[Zhao et~al.(2011)Zhao, Hachiya, Niu, and Sugiyama]{zhao2011analysis}
Zhao, T., Hachiya, H., Niu, G., and Sugiyama, M.
\newblock Analysis and improvement of policy gradient estimation.
\newblock In \emph{Advances in Neural Information Processing Systems}, pp.\
  262--270, 2011.

\end{thebibliography}
\bibliographystyle{icml2021}

\newpage
\onecolumn

\icmltitle{On-Policy Deep Reinforcement Learning for the Average-Reward Criterion Supplementary Materials}
\appendix
\section{Relationship Between the Discounted and Average Reward Criteria}\label{append:disc_avg}

We first introduce the average reward Bellman equations \citep{sutton2018reinforcement}:
\begin{align}
    &\vfunc(s) = \sum_a \pi(a|s)\left[r(s,a) - \rho(\pi) + \sum_{s'}P(s'|s,a)\vfunc(s')\right] \label{eq:bellman_vv} \\
    &\qfunc(s,a) = r(s,a) - \rho(\pi)+ \sum_{s'}P(s'|s,a)\sum_{a'}\pi(a'|s')\qfunc(s',a').\label{eq:bellman_qq}
\end{align}
From which we can easily show that:
\begin{align}
    &\vfunc(s) = \sum_a \pi(a|s)\qfunc(s,a) \label{eq:bellman_vq}\\
    &\qfunc(s,a) = r(s,a) - \rho(\pi) + \sum_{s'}P(s'|s,a)\vfunc(s').\label{eq:bellman_qv}
\end{align}
Note that these equations take a slightly different form compared to the discounted Bellman equations, there are no discount factors and the rewards are now replaced with $r(s,a)-\rho(\pi)$.

The following classic result relates the value function in the discounted case and average reward bias functions.
\begin{proposition}[\citealp{blackwell1962discrete}]\label{prop:discount_relations}
For a given stationary policy $\pi$ and discount factor $\gamma\in (0,1)$,
\begin{equation}\label{eq:disc_v}
    \lim_{\gamma\to 1}\left(\vfuncd(s) - \frac{\rho(\pi)}{1-\gamma}\right) = \vfunc(s)
\end{equation}
for all $s\in\cS$.
\end{proposition}
Note that Proposition \ref{prop:discount_relations} applies to any MDP, we will however restrict our discussion to the unichain case to coincide with the scope of the paper. From Proposition \ref{prop:discount_relations}, it is clear that $\lim_{\gamma\to 1}(1-\gamma)\rho_{\gamma}(\pi)=\rho(\pi)$, i.e. the discounted and average reward objective are equivalent in the limit as $\gamma$ approaches 1. We can derive similar relations for the action-bias function and advantage function.
\begin{corollary}\label{corollary:discount_relations}
For a given stationary policy $\pi$ and discount factor $\gamma\in (0,1)$,
\begin{align}
    &\lim_{\gamma\to 1}\left(\qfuncd(s,a) - \frac{\rho(\pi)}{1-\gamma}\right) = \qfunc(s,a) \label{eq:disc_q} \\
    &\lim_{\gamma\to 1}\advd(s,a) = \adv(s,a) \label{eq:disc_adv}
\end{align}
for all $s\in\cS$ and $a\in\cA$.
\end{corollary}
\begin{proof}
From Proposition \ref{prop:discount_relations}, we can rewrite \eqref{eq:disc_v} as
\begin{equation}
    \vfuncd(s) = \frac{\rho(\pi)}{1-\gamma} + \vfunc(s) + g(\gamma,s)
\end{equation}
where $\lim_{\gamma\to 1}g(\gamma,s)=0$. We then expand $\qfuncd(s,a)$ using the Bellman equation
\begin{align*}
    \qfuncd(s,a) &= r(s,a) + \gamma\sum_{s'}P(s'|s,a)\vfuncd(s') \\
    &= r(s,a) + \gamma\sum_{s'}P(s'|s,a)\left(\frac{\rho(\pi)}{1-\gamma} + \vfunc(s') + g^{\pi}(\gamma,s')\right) \\
    &= r(s,a) + \frac{\gamma\rho(\pi)}{1-\gamma} + \gamma\sum_{s'}P(s'|s,a)\left(\vfunc(s') + g^{\pi}(\gamma,s')\right) \\
    &= r(s,a) - \rho(\pi) + \frac{\rho(\pi)}{1-\gamma} + \sum_{s'}P(s'|s,a)\vfunc(s') \\
    & - (1-\gamma)\sum_{s'}P(s'|s,a)\vfunc(s') + \gamma\sum_{s'}P(s'|s,a)g^{\pi}(\gamma,s') \\
    &= \qfunc(s,a) + \frac{\rho(\pi)}{1-\gamma} - (1-\gamma)\sum_{s'}P(s'|s,a)\vfunc(s') + \gamma\sum_{s'}P(s'|s,a)g^{\pi}(\gamma,s') 
\end{align*}
where we used Proposition \ref{prop:discount_relations} for the second equality. Note that the last two terms in the last equality approach 0 as $\gamma\to 1$, rearranging the terms and taking the limit for $\gamma\to 1$ gives us Equation \eqref{eq:disc_q}.

We can then similarly rewrite \eqref{eq:disc_q} as
\begin{equation}
    \qfuncd(s,a) = \frac{\rho(\pi)}{1-\gamma} + \qfunc(s,a) + h(\gamma,s,a)
\end{equation}
with $\lim_{\gamma\to 1}h(\gamma, s,a)=0$. This allows us to rewrite the discounted advantage function as
\begin{align*}
    \advd(s,a) &= \qfuncd(s,a) - \vfuncd(s) \\
    &= \qfunc(s,a) + \frac{\rho(\pi)}{1-\gamma} + h^{\pi}(s,a,\gamma) - \vfunc(s)-\frac{\rho(\pi)}{1-\gamma}-g^{\pi}(s,\gamma) \\
    &= \adv(s,a) + h^{\pi}(s,a,\gamma) -g^{\pi}(s,\gamma)
\end{align*}
Since $h^{\pi}(s,a,\gamma)$ and $g^{\pi}(s,\gamma)$ both approach 0 as $\gamma$ approaches 1, taking the limit for $\gamma\to 1$ gives us Equation \eqref{eq:disc_adv}.
\end{proof}

\section{Proofs}\label{append:proof}

\subsection{Proof of Proposition \ref{prop:uninform_bound}}
\uninformbound*
\begin{proof}
We will give a proof for the case of Corollary 1 in \citet{achiam2017constrained}, a similar argument can be applied to the bound in Theorem 1 of \citet{schulman2015trust}.

We first state Corollary 1 of \citet{achiam2017constrained} which says that for any two stationary policies $\pi$ and $\pi'$:
\begin{equation}\label{eq:pol_imp_disc}
    \rho_{\gamma}(\pi')-\rho_{\gamma}(\pi) \geq \frac{1}{1-\gamma}\left[\E_{\substack{s\sim \dpid \\ a\sim\pi'}}[\advd(s,a)] - \frac{2\gamma\epsilon^{\gamma}}{1-\gamma}\E_{s\sim\dpid}\TV{\pi'}{\pi}\right]
\end{equation}
where $\epsilon^{\gamma} = \max_s\left|\E_{a\sim\pi'}[\advd(s,a)]\right|$.
Since $\dpid$ approaches the stationary distribution $\dpi$ as $\gamma\to 1$, we can multiply the right hand side of \eqref{eq:pol_imp_disc} by $(1-\gamma)$ and take the limit which gives us:
\begin{align*}
    &\lim_{\gamma\to 1}\left(\E_{\substack{s\sim \dpid \\ a\sim\pi'}}[\advd(s,a)] \pm \frac{2\gamma\epsilon^{\gamma}}{1-\gamma}\E_{s\sim\dpid}\TV{\pi'}{\pi}\right) \\
    =& \E_{\substack{s\sim \dpi \\ a\sim\pi'}}[\adv(s,a)] - 2\epsilon \E_{s\sim\dpi}[\TV{\pi'}{\pi}]\lim_{\gamma\to 1}\frac{\gamma}{1-\gamma} \\
    =& -\infty
\end{align*}
Here $\epsilon = \max_s\left|\E_{a\sim\pi'}[\adv(s,a)]\right|$. The first equality is a direct result of Corollary \ref{corollary:discount_relations}.
\end{proof}

\subsection{Proof of Lemma \ref{lemma:policy_diff}}
\policydiff*
\begin{proof}
We give two approaches for this proof. In the first approach, we directly expand the right-hand side using the definition of the advantage function and  Bellman equation, which gives us:
\begin{align*}
    \E_{\substack{s\sim d^{\pi'}\\ a\sim\pi'}}\left[\adv(s,a)\right] 
    &= \E_{\substack{s\sim d^{\pi'}\\ a\sim\pi'}}\left[\qfunc(s,a) - \vfunc(s)\right] \\
    &= \E_{\substack{s\sim d^{\pi'}\\ a\sim\pi'}}\left[r(s,a) - \rho(\pi) + \E_{s'\sim P(\cdot|s,a)}\left[\vfunc(s')\right] - \vfunc(s)\right] \\
    &= \rho(\pi') - \rho(\pi) + \E_{\substack{s\sim d^{\pi'}\\ a\sim\pi'\\ s'\sim P(\cdot|s,a)}}[\vfunc(s')] - \E_{s\sim d^{\pi'}}[\vfunc(s)]
\end{align*}
Since $\dpip(s)$ is the stationary distribution:
\begin{align*}
    \E_{\substack{s\sim d^{\pi'}\\ a\sim\pi'\\ s'\sim P(\cdot|s,a)}}[\vfunc(s')] = \sum_{s}\dpip(s)\sum_{a}\pi'(a|s)\sum_{s'}P(s'|s,a)\vfunc(s') = \sum_{s}\dpip(s)\sum_{s'}P_{\pi'}(s'|s)\vfunc(s') = \sum_{s'}\dpip(s')\vfunc(s')
\end{align*}
Therefore,
\begin{equation*}
    \E_{\substack{s\sim d^{\pi'}\\ a\sim\pi'\\ s'\sim P(\cdot|s,a)}}[\vfunc(s')] - \E_{s\sim d^{\pi'}}[\vfunc(s)] = 0
\end{equation*}
which gives us the desired result.

Alternatively, we can directly apply Proposition \ref{prop:discount_relations} and Corollary \ref{corollary:discount_relations} to Lemma 6.1 of \citet{kakade2002approximately} and take the limit as $\gamma\to 1$.
\end{proof}

\subsection{Proof of Lemma \ref{lemma:policy_impd}}

\policyimpd*

\begin{proof}
\begin{align*}
    \left|\rho(\pi') - \rho(\pi) - \E_{\substack{s\sim \dpi\\ a\sim\pi'}}\left[\adv(s,a)\right]\right| &= \left|\E_{\substack{s\sim d^{\pi'}\\ a\sim\pi'}}\left[\adv(s,a)\right] - \E_{\substack{s\sim \dpi\\ a\sim\pi'}}\left[\adv(s,a)\right]\right| \\
    &=\left|\sum_s \E_{a\sim\pi'}\left[\adv(s,a)\right]\left(\dpip(s) - \dpi(s)\right)\right| \\
    &\leq \sum_s\left|\E_{a\sim\pi'}\left[\adv(s,a)\right]\left(\dpip(s) - \dpi(s)\right)\right| \\
    &\leq \max_s\left|\E_{a\sim\pi'}\left[\adv(s,a)\right]\right|\norm{\dpip - \dpi}_1 \\
    &= 2\epsilon\TV{d^{\pi'}}{d^{\pi}}
\end{align*}
where the last inequality follows from H\"older's inequality.
\end{proof}

\subsection{Proof of Lemma \ref{lemma:d_and_pi}}
\dandpi*
\begin{proof}
Our proof is based on Markov chain perturbation theory \citep{cho2001comparison,hunter2005stationary}. Note first that
\begin{equation}\label{eq:dpi_reduce}
    \begin{aligned}
        (d_{\pi'}^T-d_{\pi}^T)(I-P_{\pi'}+P_{\pi'}^{\star}) &= d_{\pi'}^T-d_{\pi}^T - d_{\pi'}^T + d_{\pi}^T P_{\pi'} \\
    &= d_{\pi}^T P_{\pi'} - d_{\pi}^T \\
    &= d_{\pi}^T(P_{\pi'}-P_{\pi})
    \end{aligned}
\end{equation}
Right multiplying \eqref{eq:dpi_reduce} by $(I-P_{\pi'}+P_{\pi'}^{\star})^{-1}$ gives us:
\begin{equation}\label{eq:dinZ}
    d_{\pi'}^T-d_{\pi}^T = d_{\pi}^T(P_{\pi'}-P_{\pi})(I-P_{\pi'}+P_{\pi'}^{\star})^{-1}
\end{equation}
Recall that $\Zpip =( I-P_{\pi'}+P_{\pi'}^{\star})^{-1}$ and $\Mpip = (I-\Zpip + E \Zpip_{\dg})D^{\pi'}$. Rearranging the terms we find that
\begin{equation}\label{eq:ZinM}
    \Zpip = I + E \Zpip_{\dg} - \Mpip (D^{\pi'})^{-1}
\end{equation}
Plugging \eqref{eq:ZinM} into \eqref{eq:dinZ} gives us
\begin{equation}
    \begin{aligned}
    d_{\pi'}^T-d_{\pi}^T &= d_{\pi}^T(P_{\pi'}-P_{\pi})(I + E \Zpip_{\dg} - \Mpip (D^{\pi'})^{-1}) \\
    & = d_{\pi}^T(P_{\pi'}-P_{\pi})(I - \Mpip (D^{\pi'})^{-1})
    \end{aligned}
\end{equation}
where the last equality is due to $(P_{\pi'}-P_{\pi})E=0$.

Let $\norm{\cdot}_p$ denote the the operator norm of a matrix, in particular $\norm{\cdot}_1$ and $\norm{\cdot}_{\infty}$ are the maximum absolute column sum and maximum absolute row sum respectively. By the submultiplicative property of operator norms \citep{horn2012matrix}, we have:
\begin{equation}\label{eq:submult}
    \begin{aligned}
    \norm{d_{\pi'}-d_{\pi}}_1 &= \norm{(I - \Mpip (D^{\pi'})^{-1})^T (P_{\pi'}^T-P_{\pi}^T)\dpi}_1 \\
    &\leq \norm{(I - \Mpip (D^{\pi'})^{-1})^T}_1 \norm{(P_{\pi'}^T-P_{\pi}^T)\dpi}_1 \\
    &= \norm{(I - \Mpip (D^{\pi'})^{-1})}_{\infty}\norm{(P_{\pi'}^T-P_{\pi}^T)\dpi}_1
    \end{aligned}
\end{equation}
We can rewrite $\norm{I - \Mpip (D^{\pi'})^{-1}}_{\infty}$ as
\begin{equation}
    \begin{aligned}
    \norm{I - \Mpip (D^{\pi'})^{-1}}_{\infty} &= \max_{s} \left( \sum_{s'}\Mpip(s,s')\dpip(s') - 1\right) \\
    &= \kappa^{\pi'} - 1
    \end{aligned}
\end{equation}
Finally we bound $\norm{(P_{\pi'}^T-P_{\pi}^T)\dpi}_1$ by
\begin{equation}\label{eq:Pd}
    \begin{aligned}
    \norm{(P_{\pi'}^T-P_{\pi}^T)\dpi}_1 &= \sum_{s'}\left|\sum_s\left(\sum_a P(s'|s,a)\pi'(a|s)-P(s'|s,a)\pi(a|s)\right)d_{\pi}(s)\right| \\
    &\leq \sum_{s',s}\left|\sum_a P(s'|s,a)(\pi'(a|s)-\pi(a|s))\right|d_{\pi}(s) \\
    &\leq \sum_{s,s',a}P(s'|s,a)\left|\pi'(a|s)-\pi(a|s)\right|d_{\pi}(s) \\
    &\leq \sum_{s,a}\left|\pi'(a|s)-\pi(a|s)\right|d_{\pi}(s) \\
    &= 2\E_{s\sim d^{\pi}}[\TV{\pi'}{\pi}]
    \end{aligned}
\end{equation}
Plugging back into \eqref{eq:submult} and setting $\kappa^{\star}=\max_{\pi}\kappa^{\pi}$ gives the desired result.
\end{proof}

\section{Kemeny's Constant and Mixing Time}\label{append:mixing}
\begin{proposition}
Under Assumption \ref{assump:irreducible}, let $1=\lambda_1(\pi)>\lambda_2(\pi)\geq \dots\geq\lambda_{|\cS|}(\pi)> -1$ be the eigenvalues of $P_{\pi}$, we have
\begin{equation}
 \kappa^{\pi} \leq 1 + \frac{|\cS| - 1}{1 - \lambda^{\star}(\pi)}
\end{equation}
where $\lambda^{\star}(\pi) = \max_{i=2,\dots,|\cS|}|\lambda_i(\pi)|$.
\end{proposition}
\begin{proof}
For brevity, we omit $\pi$ from the notations in our proof.
Let $\lambda$ be an eigenvalue of $P$ and $u$ its corresponding eigenvector. Since $P$ is aperiodic, $\lambda\neq -1$, we then have
\begin{equation}
    \begin{aligned}
      (I-P + P^{\star})u &= u - Pu + \lim_{n\to\infty}P^n u \\
      &= (1-\lambda)u + u\lim_{n\to\infty}\lambda^n \\
      & = \left(1-\lambda + \lim_{n\to\infty}\lambda^n\right) u
    \end{aligned}
\end{equation}
where $\lim_{n\to\infty}\lambda^n=1$ when $\lambda=1$ and $0$ when $|\lambda|<1$. Therefore, $(I-P + P^{\star})$ has eigenvalues $1, 1-\lambda_2,\dots, 1-\lambda_{|\cS|}$. The fundamental matrix $Z=(I-P + P^{\star})^{-1}$ has eigenvalues $1, \frac{1}{1 - \lambda_2}, \cdots, \frac{1}{1 - \lambda_{|\cS|}}$.
We can then upper bound Kemeny's constant by
\begin{equation}
    \kappa = \Tr(Z) 
    = 1 + \sum_{i=2}^{|\cS|} \frac{1}{1 - \lambda_i} 
    \leq 1 + \sum_{i=2}^{|\cS|} \frac{1}{1 - |\lambda_i|} 
    \leq 1+  \frac{|\cS| - 1}{1 - \lambda^{\star}}
\end{equation}
\end{proof}
The expression $\lambda^{\star}(\pi)$ is called as the \emph{Second Largest Eigenvalue Modulo (SLEM)}. The Perron-Frobenius theorem says that the transition matrix $P_{\pi}$ converges to the limiting distribution $P^{\star}_{\pi}$ at an exponential rate, and the rate of convergence is determined by the SLEM (see Theorem 4.3.8 of \citet{bremaud2020markov} for more details). In fact, it turns out that the mixing time of a Markov chain is directly related to the SLEM where Markov chains with larger SLEM takes longer to mix and vice versa \cite{levin2017markov}.

\section{Average Reward Policy Improvement Bound for Aperiodic Unichain MDPs}\label{append:unichain_aperiodic}

In this section, we consider general aperiodic unichain MDPs, i.e. MPDs which satisfy Assumption \ref{assump:unichain}.

We note that Lemma \ref{lemma:policy_diff} and Lemma \ref{lemma:policy_impd} both hold under Assumption \ref{assump:unichain}. We can then show the following under the general aperiodic unichain case:

\begin{lemma}\label{lemma:unichain_d_pi}
For any aperiodic unichain MDP:
\begin{equation}
    \TV{\dpip}{\dpi} \leq \zeta^{\star} \E_{s\sim \dpi}[\TV{\pi'}{\pi}[s]]
\end{equation}
where $\zeta^{\star}=\max_{\pi}\norm{Z^{\pi}}_{\infty}$.
\end{lemma}
\begin{proof}
Note that
\begin{equation*}
    d_{\pi'}^T-d_{\pi}^T = d_{\pi}^T(P_{\pi'}-P_{\pi})(I-P_{\pi'}+P_{\pi'}^{\star})^{-1}
\end{equation*}
from Equation \ref{eq:dinZ} still holds in the general aperiodic unichain case. By the submultiplicative property, we have:
\begin{equation}
    \begin{aligned}
    \norm{d_{\pi'}-d_{\pi}}_1 &= \norm{((I-P_{\pi'}+P_{\pi'}^{\star})^{-1})^T (P_{\pi'}^T-P_{\pi}^T)\dpi}_1 \\
    &\leq \norm{((I-P_{\pi'}+P_{\pi'}^{\star})^{-1})^T}_1 \norm{(P_{\pi'}^T-P_{\pi}^T)\dpi}_1 \\
    &= \norm{(I-P_{\pi'}+P_{\pi'}^{\star})^{-1}}_{\infty}\norm{(P_{\pi'}^T-P_{\pi}^T)\dpi}_1
    \end{aligned}
\end{equation}
Using the same argument as \eqref{eq:Pd} to bound $\norm{(P_{\pi'}^T-P_{\pi}^T)\dpi}_1$ and setting $\zeta^{\star}=\max_{\pi}\norm{Z^{\pi}}_{\infty}$ gives the desired result.
\end{proof}
Combining Lemma \ref{lemma:policy_impd} and Lemma \ref{lemma:unichain_d_pi} gives us the following result:
\begin{theorem}\label{thm:pol_imp_unichain}
For any aperiodic unichain MDP, the following bounds hold for any two stochastic policies $\pi$ and $\pi'$:
\begin{align}
    \rho(\pi') - \rho(\pi) &\leq \E_{\substack{s\sim \dpi\\ a\sim\pi'}}\left[\adv(s,a)\right] + 2\tilde{\xi}\E_{s\sim \dpi}[\TV{\pi'}{\pi}[s]] \label{eq:pol_imp_upper} \\
    \rho(\pi') - \rho(\pi) &\geq \E_{\substack{s\sim \dpi\\ a\sim\pi'}}\left[\adv(s,a)\right] - 2\tilde{\xi}\E_{s\sim \dpi}[\TV{\pi'}{\pi}[s]] 
\end{align}
where $\tilde{\xi} = \zeta^{\star}\max_{s}\E_{a\sim\pi'}|\adv(s,a)|$.
\end{theorem}
The constant $\zeta^{\star}$ is always finite therefore we can similarly apply the approximate policy iteration procedure from Algorithm \ref{alg:policy_iteration} to generate a sequence of monotonically improving policies.

\section{Derivation of ATRPO}\label{append:atrpo}

In this section, we give the derivation and additional details of the ATRPO algorithm presented in Algorithm \ref{alg:atrpo}. The algorithm is similar to TRPO in the discount case but with several notable distinctions. Recall the trust region optimization problem from \eqref{eq:atrpo} where
\begin{equation*}
\begin{aligned}
  & \underset{\pol\in\Pi_{\theta}}{\text{maximize}} \quad
\E_{\substack{s\sim\dpolk\\ a\sim\pol}}[\wb{A}^{\polk}(s,a)] \\
& \text{subject to} \quad \avKL{\pol}{\polk} \leq \delta 
\end{aligned}
\end{equation*}
We once again note that the objective above is the expectation of the {\em average-reward} advantage function and not the standard discounted advantage function. As done for the derivation for discounted TRPO, we can approximate this problem by performing first-order Taylor approximation on the objective and second-order approximation on the KL constraint\footnote{The gradient and first-order Taylor approximation of $\avKL{\pol}{\polk}$ at $\theta=\theta_k$ is zero.} around $\theta_k$ which gives us: 
\begin{equation}\label{eq:atrpo_approx}
\begin{aligned}
\underset{\theta}{\text{maximize}} \quad 
&g^T(\theta-\theta_k)  \\
\text{subject to} \quad
&\frac{1}{2}(\theta-\theta_k)^T H (\theta-\theta_k)\leq\delta
\end{aligned}
\end{equation}
where
\begin{equation}
    g := \E_{\substack{s\sim\dpolk \\ a\sim\polk}}\left[\grad\log\pol(a|s)|_{\theta=\theta_k}\wb{A}^{\polk}(s,a)\right]
\end{equation}
and
\begin{equation}
    H := \E_{\substack{s\sim\dpolk \\ a\sim\polk}}\left[\grad\log\pol(a|s)|_{\theta=\theta_k}\grad\log\pol(a|s)|_{\theta=\theta_k}^T\right]
\end{equation}
Note that this approximation is good provided that the step-size $\delta$ is small. The term $g$ is the average reward policy gradient at $\theta=\theta_k$ with an additional baseline term \citep{sutton2000policy} and $H$ is the \emph{Fisher Information Matrix} (FIM) \citep{lehmann2006theory}. The FIM is a symmetrical matrix and always positive semi-definite. If we assume $H$ is always positive definite, we can solve \eqref{eq:atrpo_approx} analytically with a Lagrange duality argument which yields the solution:
\begin{equation}\label{eq:atrpo_update}
    \theta = \theta_k + \sqrt{\frac{2\delta}{g^T H^{-1}g}}H^{-1}g
\end{equation}
The update rule in \eqref{eq:atrpo_update} has the same form as that of natural policy gradients \citep{kakade2001natural} for the average reward case. Similar to discounted TRPO, both $g$ and $H$ can be approximated using samples drawn from the policy $\polk$. The FIM $H$ here is identical to the FIM $H$ for Natural Gradient and TRPO. However, the definition of $g$ is different from the definition of $g$ for discounted TRPO since it includes the average-reward advantage function. 

Thus, in order to estimate $g$ we need to estimate 
\begin{equation} \label{eq:average_expression}
    \wb{A}^{\polk}(s,a) =   \wb{Q}^{\polk}(s,a) -  \wb{V}^{\polk}(s)
\end{equation}
This can be done in various ways. One approach is to approximate the average-reward bias $\wb{V}^{\polk}(s)$ and then use a one-step TD backup (as was done in Algorithm \ref{alg:atrpo}) to estimate the action-bias function. Concretely, combining (\ref{eq:average_expression}) and the Bellman equation in \eqref{eq:bellman_qv} gives
    \begin{equation}\label{eq:bellman_objective}
        \wb{A}^{\polk}(s,a) =  r(s,a) - \rho(\polk) + \E_{s'\sim P(\cdot|s,a)}\left[\wb{V}^{\polk}(s')\right] - \wb{V}^{\polk}(s)
    \end{equation}
This expression involves the average-reward bias $\wb{V}^{\polk}(s)$, which we can approximate using a critic network $\wb{V}_{\phi_k}(s)$, giving line 7 in Algorithm \ref{alg:atrpo}. It remains to specify what the target should be for updating the critic parameter $\phi$. For this, we can similarly make use of the 
Bellman equation for the average-reward bias in Equation \eqref{eq:bellman_vv} which gives line 6 in Algorithm \ref{alg:atrpo}. Finally, like discounted TRPO, after applying the update term \eqref{eq:atrpo_update}, we use backtracking linesearch to find an update term which has a positive advantage value and also maintains KL constraint satisfaction. We also apply the conjugate gradient method to estimate $H^{-1}$.

\section{Reinforcement Learning with Average Cost Constraints}\label{append:acpo}

\subsection{The Constrained RL Problem}
In addition to learning to improve its long-term performance, many real-world applications of RL also require the agent to satisfy certain safety constraints. A mathematically principled framework for incorporating safety constraints into RL is using Constraint Markov Decision Processes (CMDP). A CMDP \citep{kallenberg1983linear,ross1985constrained,altman1999constrained} is an MDP equipped with a constraint set $\Pi_{c}$, a CMDP problems finds a policy $\pi$ that maximizes an agent's long-run reward given that $\pi\in\Pi_c$ . We consider two forms of constraint sets: the average cost constraint set $\{\pi\in\Pi:\rho_{c}(\pi)\leq b\}$ and the discounted cost constraint set $\{\pi\in\Pi:\rho_{c,\gamma}(\pi)\leq b\}$. Here $b$ is some given constraint bound, the cost constraint functions are given by
\begin{align}
    &\rho_c(\pi) := \lim_{N\to\infty}\frac{1}{N}\E_{\tau\sim\pi}\left[\sum_{t=0}^{N-1} c(s_t,a_t)\right] \\
    &\rho_{c,\gamma}(\pi) :=  \E_{\tau\sim\pi}\left[\sum_{t=0}^{\infty}\gamma^t c(s_t,a_t)\right]
\end{align}
for some bounded cost function $c:\cS\times\cA\to[c_{\min}, c_{\max}]$.
\subsection{Constrained RL via Local Policy Update}
Directly adding cost constraints to any iterative policy improvement algorithms can be sample inefficient since the cost constraint needs to be evaluated using samples from the new policy after every policy update. Instead, \citet{achiam2017constrained} proposed updating $\polk$ via the following optimization problem:
\begin{equation}\label{eq:cpo}
\begin{aligned}
 \underset{\pol\in\Pi_{\theta}}{\text{maximize}} \quad
&\E_{\substack{s\sim\dpolkd\\ a\sim\pol}}[A_{\gamma}^{\polk}(s,a)] \\
 \text{subject to} \quad & \Tilde{\rho}_{c,\gamma}(\pol) \leq b, \\ & \avKL{\pol}{\polk} \leq \delta .
\end{aligned}
\end{equation}

Here,
\begin{equation}\label{eq:disc_cost_surr}
    \Tilde{\rho}_{c,\gamma}(\pol):=  \rho_{c,\gamma}(\polk) + \frac{1}{1-\gamma}\E_{s\sim\dpolkd,a\sim\pol}\left[A_{c,\gamma}^{\polk}(s,a)\right]
\end{equation}
is a surrogate cost function used to approximate the cost constraint and $A_{c,\gamma}^{\polk}(s,a)$ is the discounted cost advantage function where we replace the reward with the cost\footnote{We can define the discounted value/action-value cost functions and the average cost bias/action-bias functions in a similar manner}. Note that \eqref{eq:disc_cost_surr} can be evaluated using samples from $\polk$. By Corollary 2 of \citet{achiam2017constrained} and \eqref{eq:tv-kl}:
\begin{equation}\label{eq:disc_surr_bound}
    \left|\rho_{c,\gamma}(\pol) - \Tilde{\rho}_{c,\gamma}(\pol)\right|\leq \frac{\gamma\epsilon_{c,\gamma}}{(1-\gamma)^2}\sqrt{2\avKL{\pol}{\polk}}
\end{equation}
where $\epsilon_{c,\gamma}=\max_s\left|\E_{a\sim\pi'}[\adv_{c,\gamma}(s,a)]\right|$ . This shows that the surrogate cost is a good approximation to $\rho_{c,\gamma}(\pol)$ when $\pol$ and $\polk$ are close w.r.t. the KL divergence. Using \eqref{eq:disc_surr_bound} and the trust region constraint, the worst-case constraint violation for when $\polkup$ is the solution to \eqref{eq:cpo} can be upper bounded (Proposition 2 of \citet{achiam2017constrained}.)

This framework is problematic when the cost constraint is undiscounted. Define the average surrogate cost as
\begin{equation}
    \Tilde{\rho}_c(\pol) := \rho_c(\polk) + \E_{\substack{s\sim\dpolk \\ a\sim\polk}}[A_c^{\pol}(s,a)]
\end{equation}
where $A_c^{\pol}(s,a)$ is the average cost advantage function. We can easily show that
\begin{equation*}
    \lim_{\gamma\to 1}(1-\gamma)(\rho_{c,\gamma}(\pol) - \Tilde{\rho}_{c,\gamma}(\pol)) = \rho_c(\pol) -  \Tilde{\rho}_c(\pol)\quad\text{and}\quad\lim_{\gamma\to 1}\frac{\gamma\epsilon_{c,\gamma}}{1-\gamma}\sqrt{2\avKL{\pol}{\polk}}=\infty
\end{equation*}
However, by Theorem \ref{thm:AvgR_policy_imp}\footnote{It is straightforward to show that the theorem still holds when we replace the reward with the cost.} and \eqref{eq:tv-kl}:
\begin{equation}
    \left|\rho_{c}(\pol) - \Tilde{\rho}_{c}(\pol)\right|\leq \xi^{\pol}_c\sqrt{2\avKL{\pol}{\polk}}
\end{equation}
where $\xi^{\pol}_c=(\kappa^{\star}-1)\max_{s}\E_{a\sim\pol}|A_c^{\polk}(s,a)|$. We then have the following result:

\begin{proposition}\label{prop:acpo_bound}
Suppose $\pol$ and $\polk$ satisfy the constraints $\Tilde{\rho}_{c}(\pol)<b$ and $\avKL{\pol}{\polk}\leq \delta$, then
\begin{equation}
    \rho_C(\pol) \leq b + \xi^{\pol}_c \sqrt{2\delta}
\end{equation}
\end{proposition}
The upper-bound in Proposition \ref{prop:acpo_bound} provides a worst-case constraint violation guarantee when $\pol$ is the solution to the average-cost variant of \eqref{eq:cpo}. It is an undiscounted parallel to Proposition 2 in \citet{achiam2017constrained} which provides a similar guarantee for the discounted case. It shows that contrary to what was previously believed \citep{tessler2018reward}, \eqref{eq:cpo}  can easily be modified to accommodate for average cost constraints and still satisfy an upper bound for worst-case constraint violation. Scalable algorithms have been proposed for approximately solving \eqref{eq:cpo}  \citep{achiam2017constrained, zhang2020first}. Proposition \ref{prop:acpo_bound} shows that these algorithms can be generalized to average cost constraints with only minor modifications. In the next section, we will show how the CPO algorithm \citep{achiam2017constrained} can be modified for average cost constraints.

\subsection{Average Cost CPO (ACPO)}\label{append:cpo}
Consider the average cost variant of \eqref{eq:cpo}:
\begin{equation}\label{eq:acpo}
\begin{aligned}
 \underset{\pol\in\Pi_{\theta}}{\text{maximize}} \quad
&\E_{\substack{s\sim\dpolk\\ a\sim\pol}}[\wb{A}^{\polk}(s,a)] \\
 \text{subject to} \quad & \Tilde{\rho}_{c}(\pol) \leq b, \\ & \avKL{\pol}{\polk} \leq \delta .
\end{aligned}
\end{equation}
Similar to TRPO/ATRPO, we can apply first and second order Taylor approximations to \eqref{eq:acpo} which then gives us
\begin{equation}\label{eq:acpo_approx}
\begin{aligned}
\underset{\theta}{\text{maximize}} \quad
&g^T(\theta-\theta_k)  \\
\text{subject to} \quad
& \Tilde{c} + \Tilde{g}^T(\theta-\theta_k)\leq 0 \\
&\frac{1}{2}(\theta-\theta_k)^T H (\theta-\theta_k)\leq\delta
\end{aligned}
\end{equation}
where $g,H$ were defined in the previous section, $\Tilde{c} = \rho(\polk) - b$, and
\begin{equation}
    \Tilde{g} := \E_{\substack{s\sim\dpolk \\ a\sim\polk}}\left[\grad\log\pol(a|s)|_{\theta=\theta_k}\wb{A}_c^{\polk}(s,a)\right]
\end{equation}
is the gradient of the constraint. Similar to the case of ATRPO, $g$, $\Tilde{g}$, $H$, and $\Tilde{c}$ can all be approximated using samples collected from $\polk$. The term $\wb{A}_c^{\polk}(s,a)$ also involves the cost bias-function (see Equation \ref{eq:bellman_objective}) which can be approximated via a separate cost critic network.
The optimization problem \eqref{eq:acpo_approx} is a convex optimization problem where strong duality holds, hence it can be solved using a simple Lagrangian argument. The update rule takes the form
\begin{equation}
    \theta = \theta_k + \frac{1}{\lambda}H^{-1}(g-\nu\Tilde{g})
\end{equation}
where $\lambda$ and $\nu$ are Lagrange multipliers satisfying \citep{achiam2017constrained}
\begin{equation}\label{eq:cpo_dual}
    \max_{\lambda,\nu\geq 0}-\frac{1}{2\lambda}\left(g^T H^{-1}g + 2\nu g^T H^{-1} \Tilde{g} + \nu^2\Tilde{g}H^{-1} \Tilde{g}^T\right) + \nu\Tilde{c} - \frac{1}{2}\lambda\delta
\end{equation}
The dual problem \eqref{eq:cpo_dual} can be solved explicitly \citep{achiam2017constrained}. Similar to ATRPO, we use the conjugate gradient method to estimate $H$ and perform a backtracking line search procedure to guarantee approximate constraint satisfaction.

\subsection{Experiment Results}

For constrained RL algorithms, being able to accurately evaluate the cost constraint for a particular policy is key to learning constraint-satisfying policies. In this section, we consider the MuJoCo agents from Section \ref{sec:experiments}. However for safety reasons, we wish the agent to maintain its average speed over a trajectory below a certain threshold which is set at 2.0 for all environments. 

Here we use the same evaluation protocol as was introduced in Section \ref{sec:eval} but we calculated the average cost (total cost / trajectory length) as well as the total return for each evaluation trajectory. We used a maximum trajectory length of 1000 for these experiments. We plotted the results of our experiments in Figure \ref{fig:cpo}. 

\begin{figure}[ht]
    \centering
    \includegraphics[width=0.8\textwidth]{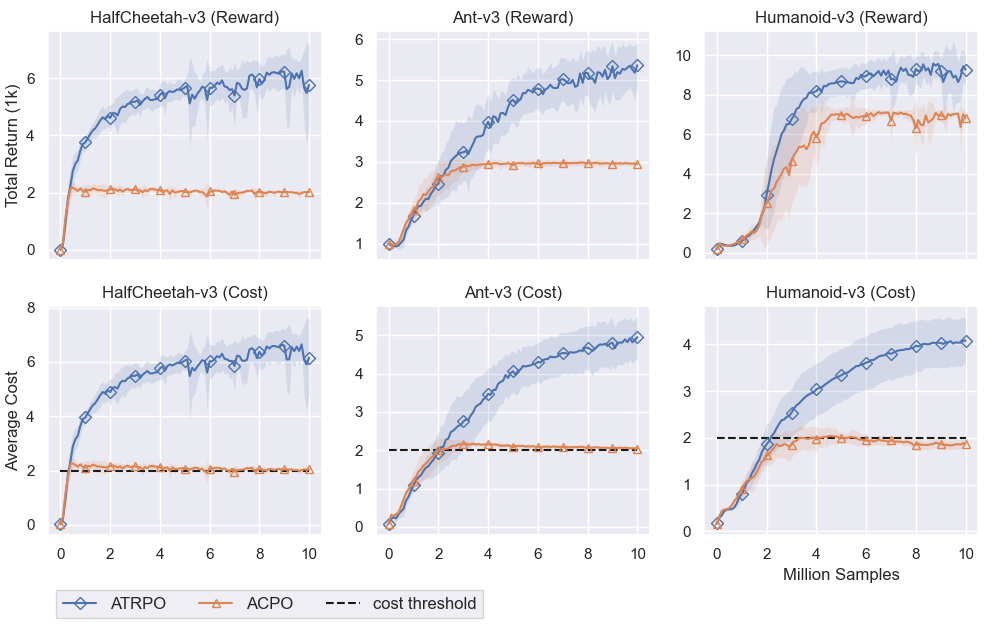}
    \caption{Performance for ACPO. Unconstrained ATRPO is plotted for comparison. The $x$-axis is the number of agent-environment interactions and the $y$-axis is the total return averaged over 10 seeds. The solid line represents the agents' average total return (top row) and average cost (bottom row) on the evaluation trajectories. The shaded region represent one standard deviation.}
    \label{fig:cpo}
\end{figure}

From Figure \ref{fig:cpo} we see that ACPO is able to learn high-performing policies while enforcing the average cost constraint.

\section{Generalized Advantage Estimator (GAE) for the Average Reward Setting}\label{append:GAE}
Suppose the agent collects a batch of data consisting of a trajectories each of length $N$ $\{s_t,a_t,r_t,s_{t+1}\}$ $(t=1,\dots,N)$ using policy $\pi$. Similar to what is commonly done for critic estimation in on-policy methods, we fit some value function $V_{\phi}^{\pi}$ parameterized by $\phi$ using data collected with the policy.

We will first review how this is done in the discounted case. Two of the most common ways of calculating the regression target for $V_{\phi}^{\pi}$ are the {\em Monte Carlo} target denoted by
\begin{equation}\label{eq:gamma_MC_target}
    V^{\text{target}}_t = \sum_{t'=t}^N \gamma^{t'-t} r_{t},
\end{equation}
or the {\em bootstrapped} target
\begin{equation}\label{eq:gamma_BS_target}
    V^{\text{target}}_t = r_{t} + \gamma\vphi(s_{t+1}).
\end{equation}

Using the dataset $\{s_{t}, V^{\text{target}}_t\}$, we can fit $V_{\phi}^{\pi}$ with supervised regression by minimizing the MSE between $V_{\phi}^{\pi}(s_{t})$ and $V^{\text{target}}_t$. With the fitted value function, we can estimate the advantage function either with the Monte Carlo estimator
\begin{equation*}
    \hat{A}^{\pi}_{\text{MC}}(s_{t}, a_{t}) = \sum_{t'=t}^N \gamma^{t'-t} r_{t} - \vphi(s_{t})
\end{equation*}
or the bootstrap estimator
\begin{equation*}
    \hat{A}^{\pi}_{\text{BS}}(s_{t}, a_{t}) = r_{t} + \gamma\vphi(s_{t+1}) - \vphi(s_{t}).
\end{equation*}
When the Monte Carlo advantage estimator is used to approximate the policy gradient, it does not introduce a bias but tends to have a high variance whereas the bootstrapped estimator introduces a bias but tends to have lower variance. These two estimators are seen as the two extreme ends of the bias-variance trade-off. In order to have better control over the bias and variance, \citet{schulman2016high} used the idea of eligibility traces \citep{sutton2018reinforcement} and introduced the Generalized Advantage Estimator (GAE). The GAE takes the form
\begin{equation}\label{eq:gamma_GAE}
    \hat{A}_{\text{GAE}}(s_{t},a_{t}) = \sum_{t'=t}^N (\gamma\lambda)^{t'-t}\delta_{t'}
\end{equation}
where
\begin{equation}
    \delta_{t'} = r_{t'} + \gamma\vphi(s_{t'+1}) - \vphi(s_{t'})
\end{equation}
and $\lambda\in[0,1]$ is the eligibility trace parameter. We can then use the parameter $\lambda$ to tune the bias-variance trade-off. It is worth noting two special cases corresponding to the bootstrap and Monte Carlo estimator:
\begin{align*}
    &\lambda = 0:\quad \hat{A}_{\text{GAE}}(s_{t},a_{t}) = r_{t} + \gamma\vphi(s_{t+1}) - \vphi(s_{t}) \\
    &\lambda = 1:\quad \hat{A}_{\text{GAE}}(s_{t},a_{t}) = \sum_{t'=t}^N \gamma^{t'-t}r_{t'} - \vphi(s_{t})
\end{align*}

For infinite horizon tasks, the discount factor $\gamma$ is used to reduce variance by downweighting rewards far into the future \citep{schulman2016high}. Also noted in \citet{schulman2016high} is that for any $l\gg 1/(1-\gamma)$, $\gamma^l$ decreases rapidly and any effects resulting from actions after $l\approx 1/(1-\gamma)$ are effectively "forgotten". This approach in essence converts a continuous control task into an episodic task where any rewards received after $l\approx 1/(1-\gamma)$ becomes negligible. This undermines the original continuing nature of the task and could prove to be especially problematic for problems where effects of actions are delayed far into the future. However, increasing $\gamma$ would lead to an increase in variance. Thus in practice $\gamma$ is often treated as a hyperparameter to balance the effective horizon of the task and the variance of the gradient estimator.

To mitigate this, we introduce how we can formulate critics for the average reward. A key difference is that in the discounted case we use $V_{\phi}^{\pi}$ to approximate the {\em discounted value function} whereas in the average reward case $\vphi$ is used to approximate the {\em average reward bias function}. 

Let
\[
\hat{\rho}_{\pi} = \frac{1}{N}\sum_{t=1}^N r_{t}
\]
denote the estimated average reward. The Monte Carlo target for the average reward value function is
\begin{equation}\label{eq:avg_MC_target}
\vtarg_t = \sum_{t'=t}^N (r_{t} - \hat{\rho}_{\pi})
\end{equation}
and the bootstrapped target is
\begin{equation}\label{eq:avg_BS_target}
\vtarg_t = r_{t} - \hat{\rho}_{\pi} +\vphi(s_{t+1}).
\end{equation}
Note that our targets (\ref{eq:avg_MC_target}-\ref{eq:avg_BS_target}) are distinctly different from the traditional discounted targets (\ref{eq:gamma_MC_target}-\ref{eq:gamma_BS_target}).

The Monte Carlo and Bootstrap estimators for the average reward advantage function are:
\begin{align*}
    \hat{A}_{\text{MC}}^{\pi}(s_{t}, a_{t}) &= \sum_{t'=t}^N (r_{t} - \hat{\rho}_{\pi}) -  \vphi(s_{t}) \\
    \hat{A}_{\text{BS}}^{\pi}(s_{t}, a_{t}) &= r_{i,t}-\hat{\rho}_{\pi} + \vphi(s_{t+1}) - \vphi(s_{t})
\end{align*}
We can similarly extend the GAE to the average reward setting:
\begin{equation}\label{eq:avg_GAE}
    \hat{A}_{\text{GAE}}(s_{t},a_{t}) = \sum_{t'=t}^N \lambda^{t'-t}\delta_{t'}
\end{equation}
where
\begin{equation}
    \delta_{t'} = r_{t'} - \hat{\rho}_{\pi}+ \vphi(s_{t'+1}) - \vphi(s_{t'}).
\end{equation}
and set the target for the value function to
\begin{equation}
    \vtarg_t = r_t - \hat{\rho}_{\pi}+ \vphi(s_{t+1}) + \sum_{t'=t+1}^N \lambda^{t'-t}\delta_{t'}
\end{equation}
The two special cases corresponding to $\lambda=0$ and $\lambda=1$ are
\begin{align*}
    &\lambda = 0:\quad \hat{A}_{\text{GAE}}(s_{t},a_{t}) = r_{t} - \hat{\rho}_{\pi} + \vphi(s_{t+1}) - \vphi(s_{t}) \\
    &\lambda = 1:\quad \hat{A}_{\text{GAE}}(s_{t},a_{t}) = \sum_{t'=t}^N (r_{t'}-\hat{\rho}_{\pi}) - \vphi(s_{t})
\end{align*}
We note again that the average reward advantage estimator is distinct from the discounted case. To summarize, in the average reward setting:
\begin{itemize}
    \item The parameterized value function is used to fit the {\em average reward bias function}.
    \item The reward term $r_{t}$ in the discounted formulation is replaced by $r_{t} - \hat{\rho}_{\pi}$.
    \item Without any discount factors, recent and future experiences are weighed equally thus respecting the continuing nature of the task.
\end{itemize}

\section{Experimental Setup}\label{append:experiment}
All experiments were implemented in Pytorch 1.3.1 and Python 3.7.4 on Intel Xeon Gold 6230 processors. We based our TRPO implementation on \url{https://github.com/ikostrikov/pytorch-trpo} and \url{https://github.com/Khrylx/PyTorch-RL}. Our CPO implementation is our own Pytorch implementation based on \url{https://github.com/jachiam/cpo} and \url{https://github.com/openai/safety-starter-agents}. Our hyperparameter selections were also based on these implementations. Our choice of hyperparameters were based on the motivation that we wanted to put discounted TRPO in the best possible light and compare its performance with ATRPO. Our hyperparameter choices for ATRPO mirrored the discounted case since we wanted to understand how performance for the average reward case differs while controlling for all other variables.

With the exception of results in Appendix \ref{append:sensitivity}, the reset cost is set to 100 on all three environments. In the original implementation of the MuJoCo environments in OpenAI gym, the maximum episode length is set to 1000\footnote{See \url{https://github.com/openai/gym/blob/master/gym/envs/__init__.py}}, we removed this restriction in our experiments in order to study long-run performance.

We used a two-layer feedforward neural network with a $\tanh$ activation for both our policy and critic networks. The policy is Gaussian with a diagonal covariance matrix. The policy networks outputs a mean vector and a vector containing the state-independent log standard deviations. States are normalized by the running mean and the running standard deviation before being fed to any network. We used the GAE for advantage estimation (see Appendix \ref{append:GAE}).  The advantage values are normalized by its batch mean and batch standard deviation before being used for policy updates. Learning rates are linearly annealed to 0 over the course of training. Note that these settings are common in most open-source implementations of TRPO and other on-policy algorithms. For training and evaluation, we used different random seeds (i.e. the random seeds we used to generate the evaluation trajectories are different from those used during training.) Table \ref{tab:hyperparameters} summarizes the hyperparameters used in our experiments.

\begin{table}[h]
    \centering
    \caption{Hyperparameter Setup}
    \vskip 0.15in
    \begin{tabular}{l l l}
    \toprule
       Hyperparameter & TRPO/ATRPO & CPO/ACPO  \\
        \midrule
    No. of hidden layers & 2 & 2  \\
      No. of hidden nodes  & 64 & 64   \\
      Activation & $\tanh$ & $\tanh$ \\
      Initial log std & -0.5 & -1 \\
      Batch size & 5000 & 5000 \\
      GAE parameter (reward) & 0.95 & 0.95  \\ 
      GAE parameter (cost) & N/A & 0.95  \\
      Learning rate for policy & $3\times 10^{-4}$ & $3\times 10^{-4}$ \\
      Learning rate for reward critic net & $3\times 10^{-4}$ & $3\times 10^{-4}$  \\
      Learning rate for cost critic net & N/A & $3\times 10^{-4}$  \\
      $L2$-regularization coeff. for critic net & $3\times 10^{-3}$ & $3\times 10^{-3}$  \\
      Damping coeff. & 0.01 & 0.01 \\
       Backtracking coeff.  & 0.8 & 0.8  \\
       Max backtracking iterations & 10 & 10 \\
       Max conjugate gradient iterations & 10 & 10 \\
      Trust region bound $\delta$ & 0.01 & 0.01  \\
       \bottomrule
    \end{tabular}
    \label{tab:hyperparameters}
\end{table}

\section{Additional Experiments}\label{append:more_exp}

\subsection{Comparing with TRPO Trained Without Resets}\label{append:trpo_mod}

Figure \ref{fig:trpo_standard} repeats the experiments presented in Figure \ref{fig:trpo} except discounted TRPO is trained in the standard MuJoCo setting without any resets (i.e. during training, when the agent falls, the trajectory terminates.) The maximum length of a TRPO training episode is 1000. This is identical to how TRPO is trained in the literature for the MuJoCo environments. We apply the same evaluation protocol introduced in Section \ref{sec:eval}. We note that when TRPO is trained in the standard MuJoCo setting, ATRPO still outperforms discounted TRPO by a significant margin.

\begin{figure}[H]
    \centering
    \includegraphics[width=0.75\textwidth]{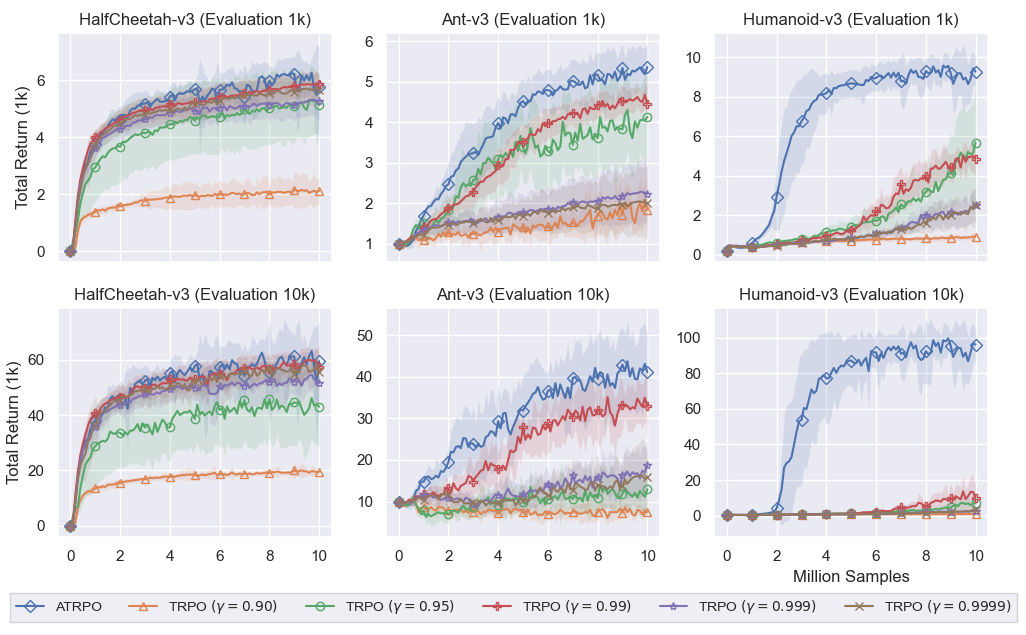}
    \caption{Comparing performance of ATRPO and TRPO with different discount factors. TRPO is trained without the reset scheme. The $x$-axis is the number of agent-environment interactions and the $y$-axis is the total return averaged over 10 seeds. The solid line represents the agents' performance on evaluation trajectories of maximum length 1,000 (top row) and 10,000 (bottom row). The shaded region represent one standard deviation.}
    \label{fig:trpo_standard}
\end{figure}

\begin{figure}[H]
    \centering
    \includegraphics[width=0.75\textwidth]{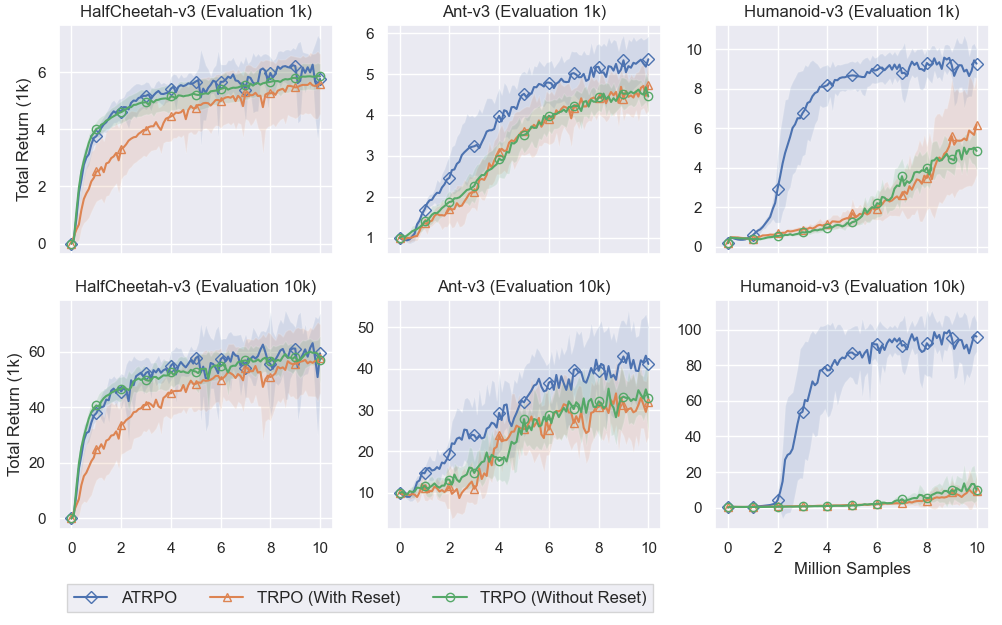}
    \caption{Comparing performance of ATRPO and TRPO trained with and without the reset costs. The curves for TRPO are for the best discount factor for each environment. The $x$-axis is the number of agent-environment interactions and the $y$-axis is the total return averaged over 10 seeds. The solid line represents the agents' performance on evaluation trajectories of maximum length 1,000 (top row) and 10,000 (bottom row). The shaded region represent one standard deviation.}
    \label{fig:trpo_reset_standard}
\end{figure}

In Figure \ref{fig:trpo_reset_standard} we plotted the performance of the best discount factor for each environment for TRPO trained with and without the reset scheme (i.e. the best performing TRPO curves from Figure \ref{fig:trpo} and Figure \ref{fig:trpo_standard}.) ATRPO is also plotted for comparison.

We note here that the performance of TRPO trained with and without the reset scheme are quite similar, this further supports the notion that introducing the reset scheme does not alter the goal of the tasks.

\subsection{Sensitivity Analysis on Reset Cost}\label{append:sensitivity}

For the experiments presented in Figure \ref{alg:atrpo}, we introduced a reset cost in order to simulate an infinite horizon setting. Here we analyze the sensitivity of the results with respect to this reset cost.

\begin{figure}[H]
    \centering
    \includegraphics[width=0.8\textwidth]{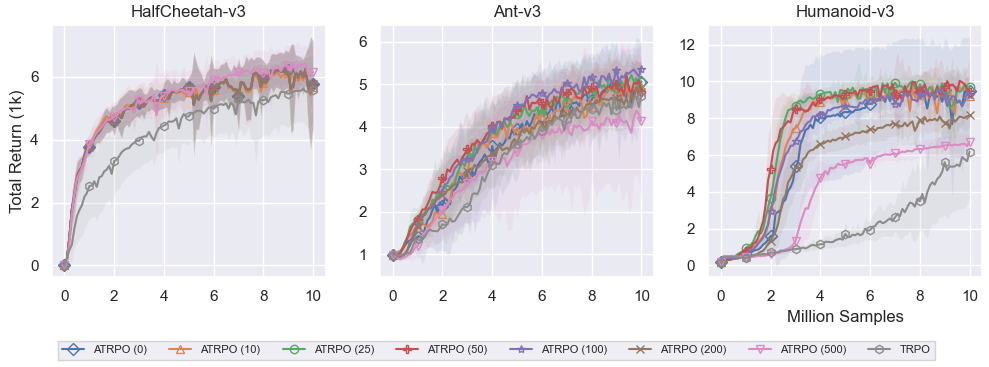}
    \caption{Comparing ATRPO trained with different reset costs to discounted TRPO with the best discount factor for each environment. The $x$-axis is the number of agent-environment interactions and the $y$-axis is the total return averaged over 10 seeds. The solid line represents the agents' performance on evaluation trajectories of maximum length 1,000. The shaded region represent one standard deviation. }
    \label{fig:trpo_sensitivity}
\end{figure}

Figure \ref{fig:trpo_sensitivity} shows that ATRPO is largely insensitive to the choice of reset cost. Though we note that for Humanoid, extremely large reset costs (200 and 500) does negatively impact performance but the result is still above that of TRPO.

\end{document}